\documentclass{article}
\usepackage{kr}

\usepackage{times}
\usepackage{soul}
\usepackage{url}
\usepackage[hidelinks]{hyperref}
\usepackage[utf8]{inputenc}
\usepackage[small]{caption}
\usepackage{graphicx}
\usepackage{amsmath}
\usepackage{amsthm}
\usepackage{mathtools}
\usepackage{booktabs}
\usepackage{algorithm}
\usepackage{algorithmic}
\newtheorem{example}{Example}

\newtheorem{definition}{Definition}

\newtheorem{proposition}{Proposition}
\newtheorem{property}{Property}
\theoremstyle{remark}
\newtheorem{remark}{Remark}

\usepackage{multirow}
\usepackage{xparse}
\usepackage{amssymb}
\usepackage{microtype}
\usepackage[capitalise,noabbrev]{cleveref}
\usepackage{todonotes}
\usepackage{thm-restate}
\usepackage{subcaption}
\usepackage{xspace}
\usepackage{pifont}
\usepackage{xcolor}
\usepackage{enumitem}
\usepackage{natbib}
\usepackage{url}
\usepackage{tikz}
\usepackage{pgfplots}
\usetikzlibrary{matrix,decorations.pathreplacing, calc, positioning,fit, shapes,decorations,intersections,patterns,arrows,shapes.arrows,
arrows.meta,trees,shapes.callouts}

\newcommand{\wrt}{w.r.t.\ }
\newcommand{\ie}{i.e.\ }

\newcommand{\cf}{cf.\ }
\newcommand{\embedding}{\ensuremath{E}\xspace}
\newcommand{\method}{\ensuremath{M}\xspace}
\newcommand{\methodembedding}[1]{{#1}-embedding}

\newcommand{\kb}{\ensuremath{\mathcal{K}}\xspace}

\newcommand{\ontology}{\ensuremath{\mathcal{T}}\xspace}
\newcommand{\pattern}{\ensuremath{\phi}\xspace}
\newcommand{\patternset}{\ensuremath{\mathcal{S}}\xspace}

\newcommand{\llang}{\mathcal{L}\xspace}
\newcommand{\abox}{\mathcal{A}\xspace}
\newcommand{\tbox}{\ontology}

\newcommand{\BoxTwoEL}{\text{Box$^2$EL}\xspace}
\newcommand{\BoxEL}{\text{BoxEL}\xspace}
\newcommand{\BoxE}{\text{BoxE}\xspace}
\newcommand{\ELEm}{\text{ELEm}\xspace}
\newcommand{\ELBE}{\text{ELBE}\xspace}
\newcommand{\ExpressivE}{\text{ExpressivE}\xspace}
\newcommand{\EmEL}{\text{EmEL$^{++}$}\xspace}
\newcommand{\loss}{\operatorname{loss}}
\newcommand{\ball}{\operatorname{Ball}}
\newcommand{\boxe}{\operatorname{Box}}
\newcommand{\bump}{\operatorname{Bump}}
\newcommand{\head}{\operatorname{Head}}
\newcommand{\taile}{\operatorname{Tail}}

\newcommand{\norm}[1]{\left\lVert#1\right\rVert}

\newcommand{\dimension}{\ensuremath{d}\xspace}
\newcommand{\complete}{entailed\xspace}
\newcommand{\cent}{\ensuremath{c}}
\newcommand{\radius}{\ensuremath{\rho}}

\newcommand{\xmark}{\ding{56}}
\newcommand{\ALC}{\ensuremath{{\cal ALC}}\xspace}
\newcommand{\EL}{\ensuremath{{\cal E\!L}}\xspace}
\newcommand{\NC}{\ensuremath{{\sf N_C}}\xspace}
\newcommand{\NI}{\ensuremath{{\sf N_I}}\xspace}
\newcommand{\NR}{\ensuremath{{\sf N_R}}\xspace}

\newcommand{\mi}[1]{\ensuremath{\mathit{#1}}}
\newcommand{\mn}[1]{\ensuremath{\mathsf{#1}}}

\newcommand{\Amc}{\ensuremath{\mathcal{A}}\xspace}

\newcommand{\Imc}{\ensuremath{\mathcal{I}}\xspace}

\newcommand{\Kmc}{\ensuremath{\mathcal{K}}\xspace}
\newcommand{\Lmc}{\ensuremath{\mathcal{L}}\xspace}

\newcommand{\Smc}{\ensuremath{\mathcal{S}}\xspace}
\newcommand{\Tmc}{\ensuremath{\mathcal{T}}\xspace}

\newcommand{\Ymc}{\ensuremath{\mathcal{Y}}\xspace}

\newcommand{\Imf}{\ensuremath{\mathfrak{I}}\xspace}

\newcommand\Tstrut{\rule{0pt}{2.4ex}}

\newcommand{\tikzcmark}{%
\tikz[scale=0.15] {
    \draw[line width=0.7,line cap=round] (0.25,0) to [bend left=10] (1,1);
    \draw[line width=0.8,line cap=round] (0,0.35) to [bend right=1] (0.23,0);
}}
\newcommand{\tikzxmark}{%
\tikz[scale=0.15] {
    \draw[line width=0.7,line cap=round] (0,0) to [bend left=6] (1,1);
    \draw[line width=0.7,line cap=round] (0.2,0.95) to [bend right=3] (0.8,0.05);
}}
\renewcommand{\xmark}{\tikzxmark}
\renewcommand{\checkmark}{\tikzcmark}

\newcommand{\new}[1]{#1}

\title{
Knowledge Base Embeddings: Semantics and Theoretical Properties
}

\author{
Camille Bourgaux$^1$\and
Ricardo Guimarães$^2$ \and
Raoul Koudijs$^2$\and
Victor Lacerda$^2$ \and
Ana Ozaki$^{2,3}$\\
\affiliations
$^1$ DI ENS, ENS, CNRS, PSL University  \& Inria, Paris, France\\
$^2$ University of Bergen\\
$^3$ University of Oslo \\
}

\begin{document}

\maketitle

\begin{abstract}
Research on knowledge graph embeddings has recently evolved into
\emph{knowledge base} embeddings, where the goal is not only to map facts into vector spaces but also constrain the models so that they take into account  the relevant conceptual knowledge available. 
This paper examines recent methods that have been proposed to embed  knowledge bases in description logic into vector spaces through the lens of their geometric-based semantics. We identify several relevant theoretical properties, which we draw from the literature and sometimes generalize or unify. We then investigate how concrete embedding methods fit in this theoretical framework.
\end{abstract}

\section{Introduction}
\label{sec:intro}
Knowledge graph (KG) embeddings allow for a continuous representation of KGs in vector spaces, which  can be used for link prediction and related tasks. 
Recent works have expanded this idea to \emph{knowledge base} (KB) embeddings, which take into account  not only   
facts but also  
conceptual knowledge, expressed as a  \emph{TBox}~\citep{Geometric,Kulmanov2019,Cone1,BoxE,EmEL,DBLP:journals/corr/abs-2202-14018,BoxEL,ExpressivE,DBLP:journals/corr/abs-2301-11118}.
Which  theoretical properties are interesting for KB embeddings? Which embedding methods have these properties? How expressive is the ontology language considered? 
These are some of the relevant questions  to  better understand how embedding methods work and which   properties they   offer.

One of the challenges to  study KB embeddings in a uniform way is that the methods differ not only in how they are defined but also in the ontology language and in the properties the authors consider. 
We focus on KBs that can be expressed in \emph{description  logic} (DL),   
and on \emph{region-based embedding methods}, which usually come with a geometric-based semantics.  
Regarding the properties, a basic goal is to determine  whether there is some kind of correspondence between classical models based on interpretations and geometric-based models created by the embedding methods. 
A simple kind of correspondence is whether the existence of a (geometric-based) model within the embedding method implies the KB is satisfiable, and vice-versa,  whether the existence of a classical interpretation that satisfies a given KB implies 
the existence of a  model within the embedding method. The former property is known as \emph{soundness} (see, e.g.,~\cite{BoxEL}) 
and we call the latter   
\emph{completeness}. 
Such correspondence does not require, for example, that
(i) axioms entailed by a given KB hold in the geometric-based model, or conversely,  that (ii) axioms that hold in the geometric-based model are a consequence of the KB  or (iii) are at least consistent with the KB.
\begin{figure}
    \includegraphics[scale=0.28]{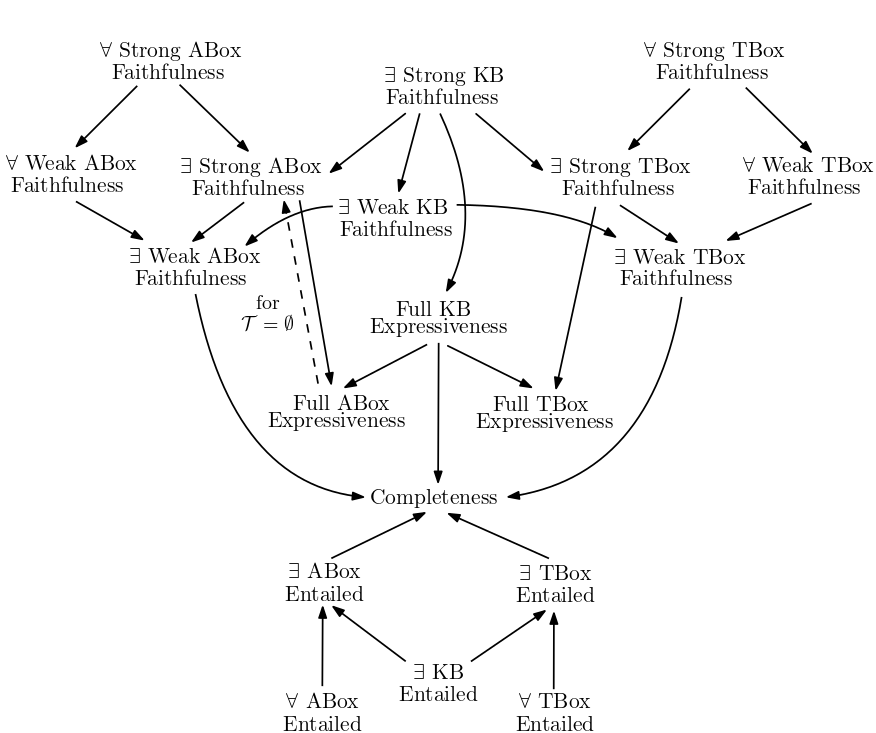}
    \caption{
    Relationships between the properties we consider (except for soundness which is incomparable \new{and KB properties  expressible by combining TBox and ABox properties)}. 
    An arrow from property X to property Y indicates that property X implies property Y. A dashed line indicates that the implication holds when the TBox $\Tmc$ is empty.
    The symbol $\forall$ defines a guarantee, while $\exists$ just posits ability. 
    }\label{fig:inflangdiag}
\end{figure}
These properties strengthen the notion of completeness. We call Property (i) 
\emph{entailment closure}, 
while Properties (ii) and (iii) correspond, respectively, to the notions of \emph{strong} and \emph{weak faithfulness} by~\cite{Cone1}. 
We study two variants for (i)-(iii): one only requires 
the \emph{ability} of an embedding method to produce a geometric-based model with the desired property, that is, whether such a model \emph{exists};
and one where the property should hold as a \emph{guarantee}, that is, in addition to ability, \emph{every}  model should have the property. 
 In the KG literature, 
 \emph{full expressiveness}~\citep{10.5555/3327144.3327341} means that, given any assignment of truth values for facts, there is an embedding model that separates true facts from false ones. We generalize this notion to include ontology languages. 
We study and formalize these different properties, 
proposing a theoretical framework for better understanding 
KB embeddings behaviour. 
 \cref{fig:inflangdiag} illustrates the relationships between the  properties. 
We also study recent KB embedding methods and 
investigate how they fit in the theoretical framework.  
Our study reveals that for many embedding methods, in particular those with an implementation, the theoretical properties stated in the literature do not hold or cannot be combined (e.g., an embedding method can be fully expressive and able to capture some patterns but not within the same  model).

We provide basic definitions in \cref{sec:preli} and present recent region-based embedding methods and their semantics in \cref{sec:embeddingmethods}. In \cref{sec:properties}, we introduce embedding method properties and show how they relate. 
We also show that if the ontology language is \emph{finite} (that is, only finitely many axioms exist in the language), which is a common assumption for KB embedding methods, then multiple properties become equivalent. 
In \cref{sec:methods-properties}, we investigate whether the embedding methods of \cref{sec:embeddingmethods} fit into the theoretical framework of \cref{sec:properties}. 
We conclude in \cref{sec:related}. 
\new{Omitted proofs are available in the appendix.} 

\begin{table}
\begin{tabular*}{\columnwidth}{l @{\extracolsep{\fill}} c l}
\toprule
\bf{Name} & \bf{Syntax} & \bf{Semantics}\\
\midrule
Top  & $\top$ & $\Delta^\Imc$\\
Bottom  & $\bot$ & $\emptyset$\\
Nominal & $\{a\}$ & $\{a^\Imc\}$\\
Negation & $\neg C$ & $\Delta^\Imc\backslash C^\Imc$\\
Conjunction & $C\sqcap D$ & $C^\Imc\cap D^\Imc$\\
Disjunction & $C\sqcup D$ & $C^\Imc\cup D^\Imc$\\
Q. exist.  res. & $\exists R.C$ & $\{d\mid (d,e)\in R^\Imc, e\in C^\Imc\}$\\
Q. univ. res. & $\forall R.C$ & $\{d\mid (d,e)\in R^\Imc\Rightarrow e\in C^\Imc\}$\\

\midrule
Inverse & $r^-$ & $\{(e,d)\mid (d,e)\in r^\Imc\}$\\
Negation & $\neg R$ & $(\Delta^\Imc \times \Delta^\Imc) \backslash R^\Imc$\\
Composition & $R \circ S$ & $\{(d,e)\mid (d,d')\in R^\Imc, (d',e)\in S^\Imc \}$\\
\bottomrule
\end{tabular*}
\caption{Syntax and semantics of common DL 
constructors: $a\in\NI$, $r\in\NR$, $C, D$ are (complex) concepts, and $R, S$ (complex) roles.}\label{constructors}
\end{table}

\section{Basic Definitions}\label{sec:preli}
This section recalls the basics of DL syntax and semantics and the basics of KB embeddings into vector spaces.

\subsection{Description Logic Knowledge Bases}\label{sec:dl-prelim}
\subsubsection*{Syntax} Let $\NC$, $\NR$, and $\NI$ be
pairwise disjoint  {finite}
sets of \emph{concept names}, \emph{role names}, and \emph{individual names} or \emph{entities}, respectively. These sets 
are usually countably infinite in the DL literature ~\citep{DlIntro} but often assumed to be finite in the KG and KB embedding literature \citep{BoxE,BoxEL}. 
 An \emph{ABox} \Amc is a finite set of concept and role assertions of the form $A(a)$ or $r(a,b)$ respectively, where $A\in\NC$, $r\in\NR$ and $a,b\in\NI$. 
 A \emph{TBox} \Tmc is a finite set of axioms whose form depends on the specific DL language. 
 Syntax of common DL constructors and TBox axioms is given in Tables \ref{constructors} and \ref{Axioms}. 
 In the KG embedding literature, 
\emph{(inference) patterns} are often considered. Table~\ref{patterns} presents these patterns and their DL translation. 
We say that a DL language \Lmc is \emph{finite} if
there are finitely many axioms expressible in \Lmc.
A DL \emph{knowledge base} $\kb=\Tmc\cup\Amc$ is the union of a TBox and an ABox. 
 
\begin{table}
\begin{tabular*}{\columnwidth}{l @{\extracolsep{\fill}} c l}
\toprule
\bf{Name} & \bf{Syntax} & \bf{Semantics}\\
\midrule
Concept inclusion & $C_1\sqsubseteq C_2$ & $C_1^\Imc\subseteq C_2^\Imc$\\
Role inclusion & $R\sqsubseteq S$ & $R^\Imc\subseteq S^\Imc$\\
\midrule
Concept assertion & $A(a)$ & $a^\Imc\in A^\Imc$\\
Role assertion & $r(a,b)$ & $(a^\Imc, b^\Imc)\in r^\Imc$\\
\bottomrule
\end{tabular*}
\caption{Syntax and semantics of common TBox and ABox axioms: $A\in\NC$, $r\in \NR$, $a, b\in\NI$, $C_1$, $C_2$ denote (complex) concepts and $R,S$ (complex) roles.}\label{Axioms}
\end{table}

\begin{table}\label{tab:patternsDLtranslation}
\begin{tabular*}{\columnwidth}{l @{\extracolsep{\fill}} c l}
\toprule
\bf{Name} & \bf{Rule form} & \bf{DL}\\
\midrule
Symmetry & $\forall \mathbf{x}\;r(x,y)\to r(y,x)$ & $r\sqsubseteq r^-$
\\
Inversion & $\forall \mathbf{x}\; r(x,y)\leftrightarrow s(y,x)$ & $r\equiv s^-$
\\
Hierarchy& $\forall \mathbf{x}\; r(x,y)\to s(x,y)$ & $r\sqsubseteq s$
\\
Intersection & $\forall \mathbf{x}\; r(x,y)\wedge s(x,y)\to t(x,y)$ & $r\sqcap s\sqsubseteq t$
\\
Composition & $\forall \mathbf{x}\;r(x,y)\wedge s(y,z)\to t(x,z)$  & $r\circ s\sqsubseteq t$
\\
Mut. exclusion & $\forall \mathbf{x}\;r(x,y)\wedge s(x,y)\to\bot$ & $r\sqsubseteq \neg s$
\\
\new{Asymmetry}
& $\forall \mathbf{x}\;r(x,y)\wedge r(y,x)\to\bot$ & $r\sqsubseteq \neg r^-$
\\
\bottomrule
\end{tabular*}
\caption{Common patterns and their DL counterparts, where $r, s , t$ are distinct roles in $\NR$ and $\forall\mathbf{x}$ is a shorthand for $\forall xy$ or $\forall xyz$. 
}\label{patterns}
\end{table}

\subsubsection*{Semantics} The semantics 
of DL KBs is given by interpretations.
An \emph{interpretation} $\Imc$ is a pair $(\Delta^\Imc ,\cdot^\Imc)$ where the \emph{interpretation domain} $\Delta^\Imc$ is a non-empty set and \emph{$\cdot^\Imc$} is a function that maps each $a\in\NI$ to some $a^\Imc\in\Delta^\Imc$, each $A\in\NC$ to some $A^\Imc\subseteq\Delta^\Imc$ and each $r\in\NR$ to some $r^\Imc\subseteq \Delta^\Imc\times\Delta^\Imc$. 
The function $\cdot^\Imc$ is extended to complex concept and roles as explained in Table \ref{constructors} and the satisfaction of TBox axioms and ABox assertions is defined by Table \ref{Axioms}. 
An interpretation \Imc is a \emph{model} of an ABox \Amc ($\Imc\models\Amc$) if it satisfies every assertion in \Amc; it is a model of a TBox \Tmc ($\Imc\models\Tmc$) if it satisfies every axiom in \Tmc; and it is a model of a KB $\kb=\Tmc\cup\Amc$ if $\Imc\models\Tmc$ and $\Imc\models \Amc$. 
A KB $\Kmc$ is \emph{satisfiable} (or \emph{consistent}) if it has a model. An axiom $\alpha$ (being an ABox assertion or a TBox axiom) is \emph{consistent with a KB $\Kmc$} if $\Kmc\cup\{\alpha\}$ is satisfiable, and is \emph{entailed} by $\kb$, 
written $\kb\models\alpha$, if $\Imc\models\alpha$ for every model $\Imc$ of $\kb$. The \emph{deductive closure} of a KB $\kb$ is the (possibly infinite) set of all axioms entailed by $\kb$.

\subsubsection*{ABoxes as TBoxes} 
Some KB embedding methods operate on the TBox 
only and 
encode the ABox into the TBox 
using nominals. Specifically, $A(a)$ is represented by $\{a\}\sqsubseteq A$ and $r(a,b)$ by $\{a\}\sqsubseteq \exists r.\{b\}$. When discussing the embedding methods properties, we still regard these axioms as assertions.

\subsection{Embedding KBs Into Vector Spaces}
\subsubsection*{Vector spaces, regions and transformations} The aim of KG or KB embedding is to learn a low-dimensional representation of the KG or KB components into some \emph{vector space(s)}.  
The \dimension-dimensional vector space $\mathbb{R}^\dimension$ is an Euclidean space whose elements are of the form $\vec{v}=(v_1,\dots,v_\dimension)$ and may be added together (using $+$) or multiplied by 
scalars (using $.$). We use $\vec{u}-\vec{v}$ as a shorthand for $\vec{u}+(-1).\vec{v}$ and say that $\vec{u}\leq\vec{v}$ if $u_i\leq v_i$ for every $1\leq i\leq \dimension$. 
The \emph{distance} between $\vec{u},\vec{v}\in \mathbb{R}^\dimension$ is the usual \emph{Euclidean distance} 
$\norm{\vec{u}-\vec{v}} =$ $\sqrt{(u_1-v_1)^2+\dots+(u_\dimension-v_\dimension)^2}$. 
Finally, $\vec{u}\oplus\vec{v}$ is the vector from $\mathbb{R}^{\dimension+\dimension'}$ that concatenates 
$\vec{u}\in \mathbb{R}^{\dimension}$ and $\vec{v}\in \mathbb{R}^{\dimension'}$.

We focus on region-based embedding methods, whose regions are usually 
convex.
A \emph{region} $X$ of $\mathbb{R}^\dimension$ is a subset of $\mathbb{R}^\dimension$. It is \emph{convex} if for every $\vec{u},\vec{v}\in X$ and $\lambda\in[0,1]$, $(1-\lambda)\vec{u}+\lambda\vec{v}$ is in $X$. 
Examples of convex regions are
\begin{itemize}
    \item \emph{convex cones}: for all $\vec{u},\vec{v}\in X$, $\lambda,\mu\geq 0$, $\lambda\vec{u}+\mu\vec{v}\in X$;
    \item \emph{boxes}: $X=\{\vec{x}\mid \vec{u}\leq\vec{x}\leq\vec{v}\}$ where $\vec{u}$ is the \emph{lower corner} and $\vec{v}$ is the \emph{upper corner} of the box; 
    \item \emph{balls}: an open (resp.\ closed) \dimension-ball of radius $\radius$ and center $\vec{x}$ is the set 
    of all $\vec{y}$ such that $\norm{\vec{y}-\vec{x}}<\radius$ (resp.\ $\leq\radius$).
\end{itemize}  
Some embedding methods rely on \emph{transformations} of $\mathbb{R}^\dimension$, which are functions $f:\mathbb{R}^\dimension\mapsto \mathbb{R}^\dimension$. An \emph{affine transformation} preserves 
convexity and parallelism and is defined by $f(\vec{x})=A \vec{x}+\vec{b}$ with $A$ an invertible matrix and $\vec{b}\in\mathbb{R}^\dimension$. If $A$ is the identity matrix, \ie 
$f(\vec{x})= \vec{x}+\vec{b}$, then $f$ is a \emph{translation}.

\subsubsection*{Embeddings} We consider abstract notions of an embedding for a DL KB and 
an embedding method. We intentionally refrain from giving more precise definitions since the existing embeddings in the literature differ so much in the way in which they embed KBs. 
An \emph{embedding} \embedding is a function that maps the components of a KB 
(such as individual, concept and role names) into abstract structures associated with vector spaces (such as regions or vector transformations). 

\begin{definition}[Embedding method]
An embedding method for $\Lmc$ is an algorithm that given an ABox and a (possibly empty) $\Lmc$-TBox, produces an embedding. We call embeddings generated by a given method \method the \methodembedding{\method}s.
\end{definition}

Embedding methods usually use
\emph{loss functions}
that penalize, e.g.,  
when regions associated 
to concepts or roles, or vectors associated with individuals, are not
placed as expected. 
Embedding methods 
optimize the loss so that the embedding captures the KB knowledge. 
The loss function often uses a \emph{margin parameter} which when less or equal to zero enforces that, for instance, the inclusion between regions is proper when the loss is zero.
\emph{Scoring functions} 
associate a score to facts or axioms, interpreted as how likely the fact or axiom is considered to be true. 
However, it is often difficult to have a fixed and pre-defined threshold
for the scoring function which is rather only used to rank facts or axioms. 

\section{KB Embeddings and Their Semantics}\label{sec:embeddingmethods}

Embeddings are usually used to assess facts or axioms (e.g., to predict plausible facts) but this can be done in different ways.  
Region-based embeddings 
come with a 
geometric-based semantics, but axioms' plausibility is also often evaluated using a scoring function, e.g., considering that an axiom is true if it gets a score above a threshold. 
 This motivates the following definition of embedding semantics, which allows for considering various semantics for a given embedding.

\begin{definition}[Embedding semantics]
A semantics for an embedding method \method is a function $S_\method$, which given an \methodembedding{\method} $\embedding$ and a language $\mathcal{L}$ 
returns a function $S_\method(\embedding, \mathcal{L})$ that maps each sentence in the language $\mathcal{L}$ 
to $1$ (meaning true) or $0$ (meaning false). 
\end{definition}

Here, we focus on region-based embedding methods and their geometric-based semantics, hence  we only consider one semantics for each 
method. Also, we  
consider one language per method.
Thus, we
may omit $S$ and \Lmc and 
write $\embedding\models_{\method}\alpha$ for $S_M(E,\Lmc)(\alpha)=1$
or $\embedding\not\models_{\method} \alpha$ for $S_M(E,\Lmc)(\alpha)=0$. 

We now briefly introduce the embedding methods we will consider in this paper, using the terminology and notation we introduced for embedding and semantics. 
Our focus is on KB embedding methods that can be applied to various DL languages but we also consider two KG embedding methods that are able to capture some patterns (\cf Table \ref{patterns}).

\subsubsection*{Convex geometric models \citep{Geometric}}
This method applies to quasi-chained rules, which include in particular the description logic $\mathcal{ELHI}_\bot$ in normal form 
(concept inclusions of the form $A\sqsubseteq B$, $A_1\sqcap A_2 \sqsubseteq B$, $\exists r^{(-)}.A\sqsubseteq B$ and $A\sqsubseteq \exists r^{(-)}.B$ with $A,A_i\in\NC\cup\{\top\}$ and $B\in\NC\cup\{\bot\}$ and role inclusions of the form $r\sqsubseteq s^{(-)}$, where $s^{(-)}$ can be a role name or its inverse). 
Each 
$a\in\NI$ is embedded as a vector $\embedding(a)\in\mathbb{R}^\dimension$, each 
$A\in\NC$ as a convex region $\embedding(A)\subseteq\mathbb{R}^\dimension$, and each 
$r\in\NR$  as a convex region $\embedding(r)\subseteq\mathbb{R}^{2\dimension}$. 
The semantics 
of this 
method for $\mathcal{ELHI}_\bot$ is given by:
\begin{itemize}
    \item $\embedding\models_{\mi{conv}} A(a)$ iff $\embedding(a)\in\embedding(A)$; 
\item $\embedding\models_{\mi{conv}} r(a,b)$ iff $\embedding(a)\oplus\embedding(b)\in\embedding(r)$; 
\item $\embedding\models_{\mi{conv}} r\sqsubseteq s^{(-)}$ iff $\embedding(r)\subseteq\embedding(s^{(-)})$; 
\item $\embedding\models_{\mi{conv}} C\sqsubseteq D$ iff $\embedding(C)\subseteq\embedding(D)$;  
\end{itemize}
where the embedding function $E$ is extended to
complex concept and role expressions in $\mathcal{ELHI}_\bot$ as follows (see \citep{DBLP:conf/dlog/BourgauxOP21} for a reference using a similar definition for a DL-Lite dialect):
\begin{itemize}
    \item $\embedding(\bot):=\emptyset$, $\embedding(\top):=\mathbb{R}^\dimension$;

\item 
$\embedding(r^-):=\{\vec{x}\oplus \vec{y}\mid \vec{x},\vec{y}\in \mathbb{R}^\dimension, \vec{y}\oplus \vec{x}\in\embedding(r)\}$;
\item $\embedding(A_1\sqcap A_2):=\embedding(A_1)\cap \embedding(A_2)$; and 
\item 
  $\embedding(\exists r^{(-)}.A):=\{\vec{x}\mid \vec{x}\in\mathbb{R}^\dimension, \vec{x}\oplus \vec{y}\in\embedding(r^{(-)}), \vec{y}\in\embedding(A)\}$. 
\end{itemize}

\subsubsection*{Al-cone models \citep{Cone1}}
\new{Since in general the complement of a convex region may not be convex, when dealing with logics with negation, it is useful to consider convex regions which have a natural 
``complementary region'' 
other than their actual complement. 
For this purpose, the authors of this method 
consider 
axis-aligned cones (al-cones), of the form $X_1\times\dots\times X_\dimension$ with $X_i\in\{\mathbb{R},\mathbb{R}_+,\mathbb{R}_-,\{0\}\}$. 
The method} 
applies to (fragments of) $\mathcal{ALC}$ (which allows for concept inclusions using the $\sqcap$, $\sqcup$, $\neg$, $\exists$ and $\forall$ constructors). 
{The authors consider
\emph{propositional \ALC} which
is a fragment of \ALC that allows Boolean connectives ($\sqcap$, $\sqcup$, $\neg$) but disallows expressions containing roles. We call this fragment $\ALC_\text{p}$.
The authors also consider the fragment of \ALC that allows concept expressions with roles but limits the size of the expressions by a constant using a notion called \emph{rank}. We denote this fragment with
$\ALC_\text{r}$.}
Each 
$a\in\NI$ is embedded as a vector $E(a)\in\mathbb{R}^\dimension\setminus\{\vec{0}\}$, each 
$A\in\NC$ as an \new{al-cone}  
$E(A)$, and each 
$r\in\NR$ as a subset $E(r)$ of $\mathbb{R}^\dimension\setminus\{\vec{0}\}\times\mathbb{R}^\dimension\setminus\{\vec{0}\}$. 
The semantics of the al-cone embedding method for $\mathcal{ALC}$ 
is defined as:  
\begin{itemize}
\item $\embedding\models_{\mi{cone}}A(a)$  iff $\embedding(a)\in\embedding(A)$; 
\item $\embedding\models_{\mi{cone}} r(a,b)$ iff $(\embedding(a),\embedding(b))\in\embedding(r)$;
\item $\embedding\models_{\mi{cone}} C_1\sqsubseteq C_2$ iff $\embedding(C_1)\subseteq \embedding(C_2)$;
\end{itemize} 
where  
the embedding function $E$ is extended to complex concepts $C_1,C_2$
as follows:
\begin{itemize}

\item 
$\embedding(C_1\sqcap C_2):=\embedding(C_1)\cap \embedding(C_2)$; 
\item $\embedding(\neg C):=\embedding(C)^o\new{=\{\vec{x}\in\mathbb{R}^\dimension\mid \forall \vec{y}\in\embedding(C), \langle\vec{x},\vec{y}\rangle\leq 0\}}$ is the polar cone of $\embedding(C)$; 
\item 
$\embedding(C_1\sqcup C_2):=\embedding(\neg(\neg C_1\sqcap \neg C_2))$; 
\item 
$\embedding(\forall r.C)$ is the minimal al-cone containing $\{\vec{x}\mid (\vec{x},\vec{y})\in \embedding(r)\Rightarrow \vec{y}\in\embedding(C)\}$; 
\item
$\embedding(\exists r.C):=\embedding(\neg\forall r.\neg C)$;   
    \item $\embedding(\top):=\mathbb{R}^\dimension$; and $\embedding(\bot):=\{\vec{0}\}$.
\end{itemize}

\begin{remark}
This semantics 
is such that it may be the case that $\embedding\not\models_{\mi{cone}} A(a)$ and $\embedding\not\models_{\mi{cone}} \neg A(a)$.
\end{remark}

\subsubsection*{\ELEm \citep{Kulmanov2019}}
This method applies to a fragment of $\EL^{++}$~\citep{DBLP:conf/owled/BaaderLB08} 
that corresponds to $\mathcal{ELO}_\bot$ (\ie $\mathcal{EL}$ with nominals and $\bot$). 
Before being embedded, ABox assertions are transformed into TBox axioms using nominals as explained in Section \ref{sec:dl-prelim} and the TBox is put in normal form. 
Each concept name or nominal $C$ is embedded as an open \dimension-ball $\embedding(C)=\ball(C)$ represented by its center $\cent(C)\in\mathbb{R}^\dimension$ and radius $\radius(C)\in\mathbb{R} $, and each 
$r\in\NR$ as a vector $\embedding(r)\in\mathbb{R}^\dimension$. The top concept \(\top\) is mapped to \(\mathbb{R}^\dimension\), that is, \(\radius(\top) = \infty\). 
    The semantics of the \ELEm embedding method based on regions is defined for $\mathcal{ELO}_\bot$ axioms in normal form below. 
\begin{itemize}
    \item For (complex) concepts $C$ and $D$ different from $\bot$, 
 $\embedding\models_\mi{elem} C\sqsubseteq D$ iff $\ball(C)\subseteq \ball(D)$ where 
 \begin{itemize}
     \item $\ball(C_1\sqcap C_2 )=\ball(C_1)\cap \ball(C_2 )$,
     \item \(\ball(\exists r.C)\) is the ball with center \(\cent(C) - \embedding(r)\) and radius \(\radius(C)\), i.e.  $\ball(\exists r.C) =\{\vec{x}\;|\;\vec{x}+E(r)\in\ball(C)\}$.
\end{itemize}
\item For concept inclusions with $\bot$ as right-hand side:
\begin{itemize}
        \item \(\embedding\models_\mi{elem} A \sqsubseteq \bot\) iff $\radius(A)=0$ (\ie $\ball(A)=\emptyset$ since $\ball(A)$ is an open ball),
        \item \(\embedding\models_\mi{elem}\exists r.A \sqsubseteq \bot\) iff $\radius(A)=0$, and 
        \item \(\embedding\models_\mi{elem} A_1\sqcap A_2 \sqsubseteq \bot\) iff $\ball(A_1)\cap\ball(A_2)\subseteq \emptyset$.
        \end{itemize}
\end{itemize}

    \subsubsection*{\EmEL \citep{EmEL}}

    \EmEL is similar to \ELEm. 
    The only difference is that \cite{EmEL} additionally consider role inclusions and role composition (hence consider the fragment $\mathcal{ELHO}(\circ)_\bot$ of $\EL^{++}$), extending the semantics as follows:
    \begin{itemize}
        \item \(\embedding\models_{\mi{emel}} r \sqsubseteq s\) iff \(E(r) = E(s)\);
        \item \(\embedding\models_{\mi{emel}} r_1\circ r_2 \sqsubseteq s\) iff \(E(r_1)+E(r_2) = E(s)\).
        \end{itemize}

\subsubsection*{\ELBE \citep{DBLP:journals/corr/abs-2202-14018}} 
\ELBE is also similar to \ELEm but uses boxes instead of balls, which has the advantage that the intersection of two boxes is still a box contrary to balls. 
Each concept name or nominal $C$ is embedded as a box $\embedding(C)=\boxe(C)$ represented by a pair of vectors $e_c(C)$ and $e_o(C)$ that represent the \emph{center} and \emph{offset} of the box. 
Specifically, the offset defines a non-negative real value for every dimension, such that 
$\vec{v}\in\boxe(C)$ iff $|\vec{v}-e_c(C)|\leq e_o(C)$. 
Each 
$r\in\NR$ is embedded as a vector $\embedding(r)\in\mathbb{R}^\dimension$. 
 We assume that the concept \(\top\) is mapped to \(\mathbb{R}^\dimension\), that is, \(e_o(\top) = \vec{\infty}\) (this is inspired by \cite{Kulmanov2019} but not explicit by \cite{DBLP:journals/corr/abs-2202-14018}).
The semantics of the \ELBE embedding method 
is defined as follows for $\mathcal{ELO}_\bot$ axioms in normal form.
\begin{itemize}
    \item For (complex) concepts $C$ and $D$ different from $\bot$, 
 $\embedding\models_\mi{elbe} C\sqsubseteq D$ iff $\boxe(C)\subseteq \boxe(D)$ where 
 \begin{itemize}
     \item $\boxe(C_1\sqcap C_2 )=\boxe(C_1)\cap \boxe(C_2 )$, \item $\boxe(\exists r.C)=\boxe(C) - \embedding(r)=\{\vec{x}\mid \vec{x}+\embedding(r)\in\boxe(C)\}$.
\end{itemize}
\item For concept inclusions with $\bot$ as right-hand side:
\begin{itemize}
        \item \(\embedding\models_\mi{elbe} A \sqsubseteq \bot\) iff $e_o(A)=\vec{0}$, and 
        \item \(\embedding\models_\mi{elbe}\exists r.A \sqsubseteq \bot\) iff $e_o(A)=\vec{0}$.
        \end{itemize}
\end{itemize}

\subsubsection*{\BoxEL \citep{BoxEL}}
\BoxEL also considers $\mathcal{ELO}_\bot$ and represents concepts as boxes, \new{but represents roles through affine transformations instead of simple translations (in contrast with \ELEm, \EmEL and \ELBE) in order to avoid that $A\sqsubseteq\exists r.B$ enforces that the volume of $B$ is at least the one of $A$ (i.e., to be able to represent many-to-one relations, that is, roles that are not inverse functional)}. 
Each 
$a\in\NI$ is mapped to a vector $\embedding(a)\in\mathbb{R}^\dimension$. 
Each $A\in\NC$ is embedded into a box $\embedding(A)=\boxe(A)$, represented by two vectors from  $\mathbb{R}^\dimension$ which give its lower and upper corners.
Each $r\in\NR$ is embedded into an affine transformation $\embedding(r)=T^r$ where $T^r(\vec{x})= D^r \vec{x} + \vec{b^r}$ with $D^r$ a diagonal matrix with non-negative entries and $\vec{b^r} \in \mathbb{R}^\dimension$. 
The semantics 
of the \BoxEL embedding method 
is given by geometric interpretations. Given an embedding $\embedding$, the corresponding geometric interpretation is $\Imc_\embedding=(\Delta^{\Imc_\embedding},\cdot^{\Imc_\embedding})$ where $\Delta^{\Imc_\embedding}=\mathbb{R}^\dimension$ and
\begin{itemize}

\item 
for every $a\in\NI$, $a^{\Imc_\embedding}:=\embedding(a)$,\item  for every $A\in\NC$, 
     $A^{\Imc_\embedding}:=\boxe(A)$, and 
     \item 
     for every 
    $r\in\NR$, $r^{\Imc_\embedding}:=\{(\vec{x},\vec{y}) 
    \mid T^r(\vec{x}) = \vec{y}\}$.  
    \end{itemize}
    We then have that $\embedding\models_{\mi{boxel}}\alpha$ iff $\Imc_\embedding\models \alpha$, where $\Imc_\embedding$ is a standard DL interpretation.

\subsubsection*{\BoxE \citep{BoxE}} 
\new{A known issue with methods that embed roles using transformations of $\mathbb{R}^\dimension$ is their inability to represent faithfully one-to-many relations, since transformations are functions. \BoxE solved this issue by introducing so called `bumps' to dynamically encode the relationship between entities and relations. It represents each relation $r$ of arity $n$ by \new{a tuple of} $n$ boxes, $\embedding(r)=(r^{(1)},\dots, r^{(n)})$, where each box $r^{(i)}$ is defined by two vectors that give its lower and upper corners, and each individual name $a\in\NI$ by two vectors, }
$\embedding(a)=(\vec{e_a}, \vec{b_a}) \in (\mathbb{R}^\dimension)^2$, where $\vec{e_a}$ is the base position of the entity and $\vec{b_a}$ its translational bump. 
The semantics of the \BoxE embedding method for the language \Lmc that consists of assertions as well as  patterns from Table \ref{patterns} except composition, is defined as follows: 
    \begin{itemize}
        \item $\embedding \models_{\mi{boxe}} A(a)$ iff $\vec{e_a}\in A^{(1)}$;
        \item $\embedding \models_{\mi{boxe}} r(c,d)$ iff $\vec{e_c}+\vec{b_d}\in r^{(1)}$ and $\vec{e_d}+\vec{b_c}\in r^{(2)}$; 
        \item $ \embedding \models_{\mi{boxe}} r_1\equiv r_2^-$ iff $r_1^{(1)}=r_2^{(2)}$ and $r_1^{(2)}=r_2^{(1)}$;
        \item $\embedding \models_{\mi{boxe}} r_1\sqsubseteq r_2$ iff $r_1^{(1)}\subseteq r_2^{(1)}$ and $r_1^{(2)}\subseteq r_2^{(2)}$; 
        \item $\embedding \models_{\mi{boxe}}  r_1\sqcap r_2\sqsubseteq r_3$ iff $r_1^{(1)}\cap r_2^{(1)}\subseteq r_3^{(1)}$ and $r_1^{(2)}\cap r_2^{(2)}\subseteq r_3^{(2)}$; 
        \item $\embedding \models_{\mi{boxe}} r_1\sqsubseteq \neg r_2$ iff $r_1^{(1)}\cap r_2^{(1)}=\emptyset$ or $r_1^{(2)}\cap r_2^{(2)}=\emptyset$;
        \item $\embedding \models_{\mi{boxe}} r_1\sqsubseteq \neg r_1^-$ iff $r_1^{(1)}\cap r_1^{(2)}=\emptyset$.
    \end{itemize}

\subsubsection*{\BoxTwoEL \citep{DBLP:journals/corr/abs-2301-11118}}
This method applies to $\mathcal{ELHO}(\circ)_\bot$, the fragment of $\EL^{++}$ also considered by \EmEL. 
It uses boxes and bumps, in line with \cite{BoxE}. Similary to \ELEm,  
ABox assertions are transformed into TBox axioms \new{with} nominals and the TBox is put in normal form. 
Each $A\in\NC$ is represented by three vectors in $\mathbb{R}^\dimension$, the first two being the lower and upper corners of a box $\boxe(A)$ and the last one defining its bump: $\embedding(A)=(\boxe(A),\bump(A))$. 
Each $a\in\NI$ is represented by a vector $\embedding(a)\in\mathbb{R}^\dimension$ and nominal $\{a\}$ is mapped to $\embedding(\{a\})=(\boxe(\{a\}),\bump(\{a\}))$ where $\boxe(\{a\})$ has volume $0$ and is such that the lower and upper corners are equal to $\embedding(a)$. 
Each  $r\in\NR$ is associated with two boxes 
$\embedding(r)=(\head(r), \taile(r))$.  
The semantics  of the \BoxTwoEL embedding method is defined for $\mathcal{ELHO}(\circ)_\bot$ axioms in normal form as follows, where given a box $B$ and a vector $\vec{v}$, $B+\vec{v}=\{\vec{x}+\vec{v}\mid \vec{x}\in B\}$ and similarly for $-$, and where a box with lower corner $\vec{l}=(l_1,\dots,l_\dimension)$ and upper corner $\vec{u}=(u_1,\dots,u_\dimension)$ is empty iff there exists $i$ such that $l_i>u_i$. 

    \begin{itemize}
     \item $\embedding\models_{\mi{box2el}} r_1\sqsubseteq r_2$ iff $\head(r_1)\subseteq \head(r_2)$ and $\taile(r_1)\subseteq \taile(r_2)$;
     \item $\embedding\models_{\mi{box2el}} r_1\circ r_2\sqsubseteq s$ iff $\head(r_1)\subseteq \head(s)$ and $\taile(r_2)\subseteq \taile(s)$;
        \item $\embedding\models_{\mi{box2el}} A\sqsubseteq B$ iff $\boxe(A)\subseteq \boxe(B)$;
        \item $\embedding\models_{\mi{box2el}} A_1\sqcap A_2\sqsubseteq B$ iff $\boxe(A_1)\cap \boxe(A_2)\subseteq \boxe(B)$;
        \item $\embedding\models_{\mi{box2el}} A\sqsubseteq \exists r.B$ iff $\boxe(A)+\bump(B)\subseteq \head(r)$ and $\boxe(B)+\bump(A)\subseteq \taile(r)$, and $\boxe(A)\subseteq \emptyset$ if $\boxe(B)=\emptyset$;
        \item $\embedding\models_{\mi{box2el}} \exists r.B\sqsubseteq A$ iff $\head(r)-\bump(B)\subseteq \boxe(A)$;
            \item $\embedding\models_{\mi{box2el}} A\sqsubseteq \bot$ iff $\boxe(A)=\emptyset$;
        \item $\embedding\models_{\mi{box2el}} A_1\sqcap A_2\sqsubseteq \bot$ iff $\boxe(A_1)\cap \boxe(A_2)=\emptyset$.
    \end{itemize}

\subsubsection*{\ExpressivE \citep{ExpressivE}}
This KG embedding method embeds each $a\in\NI$ as a vector $E(a)\in\mathbb{R}^\dimension$ and each $r\in\NR$ as an hyper-parallelogram in the virtual triple space $\mathbb{R}^{2\dimension}$ (more precisely, $r$ is mapped to three vectors from $\mathbb{R}^\dimension$: a slope, a center and a width vector). 
The semantics 
of the \ExpressivE embedding method for the language 
of role assertions and  patterns from Table \ref{patterns} is defined as follows.
\begin{itemize}
\item $\embedding\models_{\mi{expr}} r(a,b)$ iff $\embedding(a)\oplus\embedding(b)\in\embedding(r)$;
\item $\embedding\models_{\mi{expr}} r_1\sqsubseteq r_1^-$ iff $\embedding(r_1)$ is symmetric (\ie is its own mirror image w.r.t.\ the identity line);
\item $\embedding\models_{\mi{expr}} r_1\equiv r_2^-$ iff $\embedding(r_1)$ and $\embedding(r_2)$ are mirror images of each other w.r.t.\ the identity line;
\item $\embedding\models_{\mi{expr}} r_1\sqsubseteq r_2$ iff $\embedding(r_1)\subseteq\embedding(r_2)$;
\item $\embedding\models_{\mi{expr}} r_1\sqcap r_2\sqsubseteq r_3$ iff $\embedding(r_1)\cap\embedding(r_2)\subseteq\embedding(r_3)$;
\item $\embedding\models_{\mi{expr}} r_1 \sqsubseteq\neg r_2$ iff $\embedding(r_1)\cap\embedding(r_2)=\emptyset$;
\item $\embedding\models_{\mi{expr}} r_1\sqsubseteq\neg r_1^-$ iff $\embedding(r_1)$ does not intersect with its mirror image;
\item $\embedding\models_{\mi{expr}} r_1\circ r_2\sqsubseteq r_3$ iff $E(r_1\circ r_2)\subseteq\embedding(r_3)$ where $E(r_1\circ r_2)$ is the compositionally defined convex region of $r_1$ and $r_2$, which is such that, 
for every $\vec{u},\vec{v},\vec{w}\in\mathbb{R}^\dimension$, 
$\vec{u}\oplus \vec{v}\in\embedding(r_1)$ and $\vec{v}\oplus \vec{w}\in\embedding(r_2)$ iff $\vec{u}\oplus\vec{w}\in E(r_1\circ r_2)$.
\end{itemize}
\new{We point out that the works by \cite{Geometric} and \cite{Cone1} have focused on theoretical aspects of their methods, without providing an implementation. The authors of the other embedding methods we describe above have  provided implementations.
}

\new{
Other embedding methods have been designed for DLs in different contexts. For example, CosE \citep{app122010690} embeds a DL-Lite$_\mn{core}$ TBox,  
seen as a KG with relations \mn{subClassOf} and \mn{disjointWith} to find plausible missing inclusion or disjointness between concepts. 
Closer to the methods we consider, 
TransOWL and TransROWL \citep{DBLP:conf/esws/dAmatoQF21} 
were proposed for injecting background knowledge, which can be seen as a TBox.  
The basic idea is to modify the loss function by considering the facts (both positive and negative) inferred from those observed and the background knowledge. 
However, this method does not associate regions to 
concepts or roles, thus is out of the scope of this work. 
For a wider overview of possibly non-region based geometrical embeddings,
see~\citep{DBLP:journals/corr/abs-2304-11949}.
}

\section{Embedding Method Properties}\label{sec:properties}
We formulate theoretical properties for KB embeddings and embedding methods, show how they relate to each other and illustrate them on the embedding methods presented in Section \ref{sec:embeddingmethods}. 
In this section, \Lmc denotes a DL language, $\method$ is an embedding method for \Lmc, and $S_\method$ is a semantics for $\method$. 

\begin{definition}[\method-model]
Let \Amc be an ABox, \ontology be a TBox in $\Lmc$ and \embedding be an \methodembedding{\method}. 
The embedding \embedding interpreted under $S_\method$ is an $\method$-\emph{model} of 
\new{
\begin{itemize}
    \item \Amc if for every fact $\alpha$ of \Amc, $S_\method(\embedding,\mathcal{L})(\alpha)=1$, 
    \item \ontology if for every axiom $\alpha$ of $\ontology$, $S_\method(\embedding,\mathcal{L})(\alpha)=1$, 
    \item $\Kmc=\Tmc\cup\Amc$ if it is an $\method$-model
    of $\Amc$ and $\Tmc$. 
\end{itemize}
}
\end{definition}

The existence of an $\method$-\emph{model} does not 
imply the 
existence of a model in the classical sense. That is, nothing in the definition of an \method-model prevents $S_\method(\embedding,\mathcal{L})$ to assign to true inconsistent sets of axioms or to assign to false axioms that are entailed by axioms assigned to true. 

\begin{example}\label{ex:elem-not-sound} 
Consider the (classically) unsatisfiable KB $\Kmc=\Tmc\cup\Amc$ with $\Tmc=\{A\sqsubseteq \bot \}$ and $\Amc=\{A(a)\}$. 
Define an \ELEm-embedding $\embedding$ of $\{A\sqsubseteq \bot, \{a\}\sqsubseteq A \}$ in $\mathbb{R}^2$ as follows: $\embedding(\{a\})=\embedding(A)=\ball(A)$ with center $\cent(A)=(0, 1)$ and radius $\radius(A)=0$, \ie $\embedding(\{a\})=\embedding(A)=\emptyset$. 
It holds that $\embedding\models_{\mi{elem}} A\sqsubseteq \bot$ and $\embedding\models_{\mi{elem}} \{a\}\sqsubseteq A$ so $\embedding$ is an \ELEm-model of $\Kmc$. 
\EmEL   encounters the same problem since it translates 
assertions into TBox axioms then treats nominals 
as concepts so that they can be embedded to empty balls. 
\BoxEL fixed this issue by mapping individuals to vectors.
\end{example}

Conversely, non-existence of an 
\method-model also does not imply
non-existence of a model (in the classical sense). 

\begin{example}\label{ex:convex-not-complete}

As explained by \citeauthor{Geometric} (\citeyear{Geometric}), the following KB does not have any convex geometric model while it is satisfiable: $\ontology=\{r_1\sqsubseteq \neg r_2\}$
and $\Amc=\{r_1(a,b),r_1(b,a),r_2(a,a),r_2(b,b)\}$. Indeed, if $\embedding$ is a convex geometric model of $\Amc$, the following holds:
\begin{itemize}
    \item $\embedding(a)\oplus\embedding(b)\in\embedding(r_1)$ and $\embedding(b)\oplus\embedding(a)\in\embedding(r_1)$ so that by convexity, $0.5(\embedding(a)\oplus\embedding(b))+0.5(\embedding(b)\oplus\embedding(a))\in\embedding(r_1)$;
    \item $\embedding(a)\oplus\embedding(a)\in\embedding(r_2)$ and $\embedding(b)\oplus\embedding(b)\in\embedding(r_2)$ so that by convexity, $0.5(\embedding(a)\oplus\embedding(a))+0.5(\embedding(b)\oplus\embedding(b))\in\embedding(r_2)$.
\end{itemize}
Let $\vec{v}=0.5\embedding(a)+0.5\embedding(b)$.
It holds that $\vec{v}\oplus\vec{v}$ is both in $\embedding(r_1)$ and $\embedding(r_2)$, so $\embedding\not\models_{\mi{conv}} r_1\sqsubseteq \neg r_2$.    
\end{example}

\subsection{Soundness and Completeness}
This section is concerned with 
the relationship between the existence of an \method-model and that of a classical model.

\begin{property}[Embedding method soundness]\label{prop:soundness}
We say that \method under  $S_\method$ is \emph{sound for $\mathcal{L}$} if the existence of an \method-model (under $S_\method$) for a  KB $\kb$ in \Lmc implies that $\kb$ is satisfiable. 
\end{property}

\begin{property}[Embedding method completeness]\label{prop:completeness}
We say that \method under   $S_\method$ is \emph{complete for $\Lmc$} if for every satisfiable KB $\kb$ in \Lmc, there is an \method-model (under $S_\method$) for $\kb$.
\end{property}

\begin{example}
Corollary 1 in \citep{Geometric} states that embedding methods that produce convex geometric models are sound and complete for the language of quasi-chained rules (hence in particular for $\mathcal{ELHI}_\bot$ in normal form), and Proposition 2 in \citep{Cone1} states that methods that produce al-cones models are sound and complete for $\ALC_\text{p}$. 
\end{example}

As recalled in Example \ref{ex:convex-not-complete}, 
embedding methods that produce
convex geometric models
are not complete for languages 
with role disjointness, under 
a semantics where role disjointness means disjointness of the role embeddings. 
\cref{ex:boxeincomplete} shows that \BoxE (which does not fall into this class) 
is also incomplete for languages with role disjointness.

\begin{example}\label{ex:boxeincomplete}
    Consider the satisfiable KB $\Kmc=\Tmc\cup\Amc$ with $\Amc=\{r(a,b),s(a,c),r(d,c),s(d,b)\}$ and $\Tmc=\{r\sqsubseteq \neg s\}$. 
    Assume for a contradiction that there exists a \BoxE-model $\embedding$ of $\Kmc$. 
    Recall that $\embedding$ maps 
    each role $r$ to two boxes, 
    one for the ``head'', denoted $r^{(1)}$, and one for the ``tail'', 
    denoted $r^{(2)}$. 
    Also, recall that each box $r^{(i)}$ is represented by its lower and upper corners, 
    denoted $\vec{l_{r^{(i)}}}$ and $\vec{u_{r^{(i)}}}$ respectively. Moreover, a point  $\vec{e}$ is in a box $r^{(i)}$ if it is between its lower and upper corners, in symbols, $\vec{l_{r^{(i)}}}\leq\vec{e}\leq\vec{u_{r^{(i)}}}$.
    
    Since $\embedding\models_{\mi{boxe}} r\sqsubseteq \neg s$, then 
    $r^{(1)}\cap s^{(1)}=\emptyset$ or 
    $r^{(2)}\cap s^{(2)}=\emptyset$. 
    Assume $r^{(1)}\cap s^{(1)}=\emptyset$ (the argument for the  case where $r^{(2)}\cap s^{(2)}=\emptyset$ is analogous).
    Given a vector $\vec{v}$, denote by $\vec{v}[k]$ the value of $\vec{v}$ at position $k$. 
    As $r^{(1)}\cap s^{(1)}=\emptyset$, there is a dimension $j$ such that \[\vec{u_{r^{(1)}}}[j] < \vec{l_{s^{(1)}}}[j] \text{ or } \vec{u_{s^{(1)}}}[j] < \vec{l_{r^{(1)}}}[j].\]
    Suppose $\vec{u_{r^{(1)}}}[j] < \vec{l_{s^{(1)}}}[j]$. As $\embedding\models_{\mi{boxe}}r(a,b)$ and $\embedding\models_{\mi{boxe}}s(a,c)$, it must be the case that
    \[  \vec{e_a}[j]+\vec{b_b}[j] \leq\vec{u_{r^{(1)}}}[j]<\vec{l_{s^{(1)}}}[j]\leq \vec{e_a}[j]+\vec{b_c}[j]  \]
which implies that $\vec{b_b}[j]<\vec{b_c}[j]$. Now, as $\embedding\models_{\mi{boxe}}r(d,c)$ and $\embedding\models_{\mi{boxe}}s(d,b)$, we obtain 
     \[ \vec{e_d}[j]+\vec{b_c}[j] \leq\vec{u_{r^{(1)}}}[j]<\vec{l_{s^{(1)}}}[j]\leq \vec{e_d}[j]+\vec{b_b}[j]  \]
which implies that $\vec{b_c}[j]<\vec{b_b}[j]$, contradicting $\vec{b_b}[j]<\vec{b_c}[j]$. The case $\vec{u_{s^{(1)}}}[j] < \vec{l_{r^{(1)}}}[j]$ can be proved analogously.
\end{example}

In the literature, it is common to consider an alternative meaning for soundness, which intuitively links the existence of an embedding with loss 0 and KB satisfiability. 
A loss function associated with an  embedding method \method can be seen as a function $\loss$ that takes as input a KB $\kb$ and an \methodembedding{\method} \embedding of $\kb$ and returns a number. 

\begin{property}[Embedding method soundness based on loss]
If \method  has a loss function $\loss$, we say that  \method 
is \emph{sound for $\mathcal{L}$ \wrt the loss function} if the existence of an \methodembedding{\method} \embedding of a KB $\kb$ in \Lmc 
 such that $\loss(\kb,\embedding)=0$ 
implies that $\kb$ is satisfiable. 
\end{property}

\begin{example}
Example \ref{ex:elem-not-sound} shows that \ELEm is not sound. Moreover, it also shows that \ELEm is not sound \wrt the loss function defined in \citep{Kulmanov2019}. If the margin parameter $\gamma$ is equal to $0$ 
since $\embedding(\{a\})$ and $\embedding(A)$ have the same center  that lies on the unity sphere and the same radius $0$, the loss of the axiom $\{a\}\sqsubseteq A$ given by $\max(0,\lVert \cent(\{a\}) -\cent(A)\rVert+\radius(\{a\})-\radius(A)-\gamma)+|\lVert\cent(\{a\})\rVert -1|+|\lVert\cent(A)\rVert -1|$ is equal to 0, and the loss of the axiom $A\sqsubseteq\bot$ given by $\radius(A)$ is equal to 0. Hence, $\loss(\kb,\embedding)=0$. 
For Theorem 1 in \citep{Kulmanov2019} to hold, $\gamma$ should be strictly negative, rather than~$\leq 0$.\footnote{
With a minor fix in~\citep[Equation 2]{Kulmanov2019}: one of the subterms of the loss term for 
\(A_1 \sqcap A_2 \sqsubseteq B\) is 
$\max(0, \norm{\cent(A_1) - \cent(B)} - \radius(A_1) - \gamma)$. 
The following subterm is written as $\max(0, \norm{\cent(A_2) - \cent(B)} - \radius(A_1) - \gamma)$ while it should be analogous to the one before, using   \(\radius(A_2)\) instead of \(\radius(A_1)\).
} 
This 
however prevents \ELEm to embed equivalent concepts, such as $\{A\sqsubseteq B,B\sqsubseteq A\}$, with loss 0.
\end{example}

\cref{ex:Box2ELNotSound} illustrates the difference between soundness \wrt the loss function and soundness as in~Property~\ref{prop:soundness}.
\begin{example}\label{ex:Box2ELNotSound}
Theorem 1 in \citep{DBLP:journals/corr/abs-2301-11118} shows that \BoxTwoEL is sound \wrt the loss function. However, \BoxTwoEL is not sound. Indeed, the loss function of \BoxTwoEL is such that in models of loss 0,  all bumps are equal to $\vec{0}$, while for soundness, we consider also models with non-zero bumps. 
To illustrate this, consider $\Tmc=\{\exists r.B\sqsubseteq A, \exists s.C\sqsubseteq D, A\sqcap D\sqsubseteq \bot\}$ and $\Amc=\{r(a,b), s(a,c), B(b), C(c)\}$. The KB $\Kmc=\Tmc\cup\Amc$ is unsatisfiable. 
However, the \BoxTwoEL-embedding depicted in Figure \ref{Figure:Box2ELNotSound} is a \BoxTwoEL-model of $\Kmc$. Note that \BoxTwoEL-embeddings with loss 0 have an   undesirable behaviour: since all bumps are $\vec{0}$, $\embedding\models_{\mi{box2el}}A\sqsubseteq\exists r.B$ and $\embedding\models_{\mi{box2el}}C\sqsubseteq\exists r.D$ imply $\embedding\models_{\mi{box2el}}A\sqsubseteq\exists r.D$ and $\embedding\models_{\mi{box2el}}C\sqsubseteq\exists r.B$.
\end{example}
\begin{figure}
\centering
    \includegraphics[width=0.35\textwidth,  keepaspectratio]{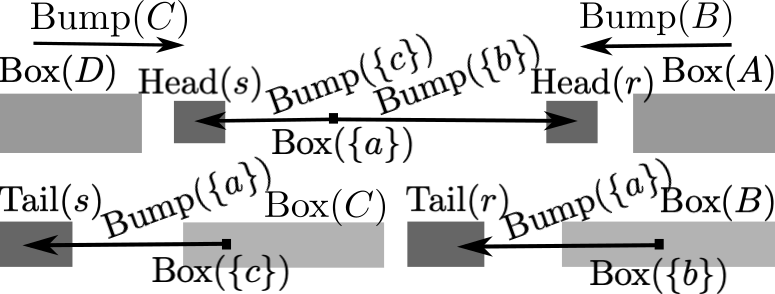}
{\footnotesize
\begin{align*}
\exists r.B\sqsubseteq A:&\quad\head(r)-\bump(B)\subseteq \boxe(A)\\
\exists s.C\sqsubseteq D:&\quad\head(s)-\bump(C)\subseteq \boxe(D)\\
 A\sqcap D\sqsubseteq \bot:&\quad \boxe(A)\cap \boxe(D)=\emptyset\\
r(a,b):&\quad \left\{\begin{matrix*}[l]\boxe(\{a\})+\bump(\{b\})\subseteq \head(r)\\
\boxe(\{b\})+\bump(\{a\})\subseteq \taile(r)\end{matrix*}\right\}\\
s(a,c):&\quad\left\{\begin{matrix*}[l]\boxe(\{a\})+\bump(\{c\})\subseteq \head(s)\\
\boxe(\{c\})+\bump(\{a\})\subseteq \taile(s)
\end{matrix*}\right\}
\end{align*}
$B(b): \boxe(\{b\})\subseteq \boxe(B)\quad C(c): \boxe(\{c\})\subseteq \boxe(C)$
}
\caption{\BoxTwoEL-embedding.}
\label{Figure:Box2ELNotSound}
\end{figure}

\subsection{Entailment Closure and Faithfulness}
Since \method-models come with very few guarantees on what they assign to true besides the KB itself, additional properties can be required on the \method-models. 
Entailment closure  
guarantees that all consequences of the KB are assigned to true. 

\begin{definition}[Entailment closure in an \method-model]\label{def:xmodelcompleteness}
Let \ontology be a TBox in \Lmc and \Amc be an ABox such that $\Kmc=\Tmc\cup\Amc$ is satisfiable. We say that an $M$-model $\embedding$ of $\Kmc$ is
\begin{itemize}
\item \emph{TBox-\complete} for \Lmc if for every TBox axiom $\alpha$ in $\mathcal{L}$ that is entailed by \new{$\Kmc$}, $\embedding\models_\method \alpha$; 
\item \emph{ABox-\complete} if for every assertion $\alpha$ 
that is entailed by $\Kmc$, $\embedding\models_\method \alpha$;
\item \new{\emph{KB-entailed} if it is TBox-entailed and ABox-entailed.}
\end{itemize}
\end{definition}

A slight modification of \cref{ex:Box2ELNotSound} provides a \BoxTwoEL-model that is not ABox-entailed.

\begin{example}
Consider $\Tmc=\{\exists r.B\sqsubseteq A, \exists s.C\sqsubseteq D\}$ and $\Amc=\{r(a,b), s(a,c), B(b), C(c)\}$. The KB $\Kmc=\Tmc\cup\Amc$ is satisfiable and $\Kmc$ entails $ A(a)$ and $D(a)$. 
However, the \BoxTwoEL-model of $\Kmc$ in Figure \ref{Figure:Box2ELNotSound} does not satisfy $A(a)$, $D(a)$. 
\end{example}

Entailment closure does not prevent the embedding semantics to assign to true axioms that are, e.g., not consistent with the KB. 
The notions of weak and strong faithfulness have been proposed in the literature and address this issue.

\begin{definition}[Weak faithfulness of an \method-model (adapted from \citep{Cone1})] 
\label{def:weakfaithful}
Let \ontology be a TBox in $\mathcal{L}$ and let \Amc be an ABox such that $\Kmc=\Tmc\cup\Amc$ is satisfiable. We say that an $M$-model $\embedding$ of $\Kmc$ is
\begin{itemize}
\item \emph{weakly TBox-faithful} for \Lmc if for every TBox axiom $\alpha$ in   $\mathcal{L}$, $\embedding\models_\method \alpha$ 
implies that $\alpha$ is consistent with \new{$\Kmc$};
\item \emph{weakly ABox-faithful} if for every assertion $\alpha$, $\embedding\models_\method \alpha$ 
implies that $\alpha$ is consistent with $\Kmc$;
\item \new{\emph{weakly KB-faithful} if it is weakly TBox-faithful and weakly ABox-faithful.}
\end{itemize}   
\end{definition}

\cref{ex:notweakfaith} shows that some KBs may have only \ELEm- or \EmEL-models that are not weakly ABox-faithful.
\begin{example}\label{ex:notweakfaith}
Let $\Tmc=\{B\sqcap C\sqsubseteq\bot\}$ and $\Amc=\{r(a,b),r (a,c), B(b), C(c)\}$. 
For every $\method$-model $\embedding$ of $\Amc$ and $\Tmc$ with $\method\in\{\ELEm,\EmEL\}$, $\embedding\models_{\method} C(a)$ and $\embedding\models_{\method} B(a)$. Indeed, since $\embedding\models_{\method} \{a\}\sqsubseteq \exists r.\{b\}$, it holds that $\embedding(\{a\})\subseteq \embedding(\{b\})-\embedding(r)$ and similarly, $\embedding(\{a\})\subseteq \embedding(\{c\})-\embedding(r)$. Since $\embedding(\{b\})\subseteq \embedding(B)$, and $\embedding(\{c\})\subseteq \embedding(C)$, 
it follows that $\embedding(\{a\})+\embedding(r)\subseteq \embedding(B)\cap \embedding(C)=\emptyset$, i.e., $\embedding(\{a\})=\emptyset$ is included in every region of $\mathbb{R}^\dimension$. 
\end{example}

A stronger condition than weak faithfulness ensures that models satisfy \emph{only} the KB consequences. 

\begin{definition}[Strong faithfulness of an \method-model (adapted from \citep{Cone1})]\label{def:strongfaithful} 
Let \ontology be a TBox in \Lmc and 
\Amc an ABox such that $\Kmc=\ontology\cup\Amc$ is satisfiable. We say that an $M$-model $\embedding$ of $\Kmc$ is

\begin{itemize}
\item \emph{strongly TBox-faithful} for \Lmc if for every TBox axiom $\alpha$ in   $\mathcal{L}$, $\embedding\models_\method \alpha$ 
implies that $\alpha$ is entailed by \new{$\Kmc$};
\item \emph{strongly ABox-faithful} if, for every assertion $\alpha$, $\embedding\models_\method \alpha$ 
implies that  $\alpha$ is entailed by $\Kmc$;
\item \new{\emph{strongly KB-faithful} if it is strongly TBox-faithful and strongly ABox-faithful.}
\end{itemize}  
\end{definition}

\cref{ex:not-strongly-faithfuk-model} illustrates that some KBs may have only \BoxTwoEL-models that are not strongly TBox-faithful. 
\begin{example}\label{ex:not-strongly-faithfuk-model}
Consider $\Tmc=\{r_1\circ r_2\sqsubseteq r_3, \exists r_3.C\sqsubseteq D\}$. 
Let $\embedding$ be a \BoxTwoEL-model of $\Tmc$. 
Since $\embedding\models_{\mi{box2el}}r_1\circ r_2\sqsubseteq r_3$, then $\head(r_1)\subseteq \head(r_3)$, and since $\embedding\models_{\mi{box2el}} \exists r_3.C\sqsubseteq D$, then $\head(r_3)-\bump(C)\subseteq \boxe(D)$. It follows that $\head(r_1)-\bump(C)\subseteq \boxe(D)$, so $\embedding\models_{\mi{box2el}} \exists r_1.C\sqsubseteq D$.  However, $\Tmc\not\models \exists r_1.C\sqsubseteq D$. 
\end{example}

\cref{ex:elemmodelnotstronglytboxfaith} shows that some KBs may have only \ELEm- or \EmEL-models that are not strongly TBox-faithful.
\begin{example}\label{ex:elemmodelnotstronglytboxfaith}
Let $\Tmc=\{\exists r.C\sqsubseteq A,\exists r.D\sqsubseteq B, A\sqcap B\sqsubseteq \bot\}$. 
For every $\method$-model $\embedding$ of $\Tmc$ with $\method\in\{\ELEm,\EmEL\}$, $\embedding(C)-\embedding(r)\subseteq \embedding(A)$, $\embedding(D)-\embedding(r)\subseteq \embedding(B)$ and $\embedding(A)\cap \embedding(B)\subseteq \emptyset$. Hence $\embedding(C)\cap \embedding(D)\subseteq \emptyset$. 
It follows that $\embedding\models_{\method} C\sqcap D\sqsubseteq \bot$ while $\Tmc\not\models C\sqcap D\sqsubseteq \bot$.
\end{example}

\begin{remark}
    The dimension of the embedding space $\mathbb{R}^\dimension$ may have a strong impact on strong TBox faithfulness. Indeed, if \Lmc is a language allowing for concept intersections, and \method is an embedding method which maps concepts to convex regions, $\bot$ to $\emptyset$, and interprets the conjunction of concepts as the intersection of their embeddings, then for every $k > \dimension+1$, the TBox $\tbox = \{C_1 \sqcap ... \sqcap C_k \sqsubseteq \bot\}$ is such that no \method-model of $\tbox$ is strongly TBox-faithful. Indeed, if $\embedding$ is an $\method$-model of $\tbox$, then $\bigcap_{i=1}^k\embedding(C_i)=\emptyset$ and Helly's theorem~\citep{Helly1923} states that if 
    $X_1,\dots, X_k$ are convex regions in $\mathbb{R}^\dimension$, with $k>\dimension$, and each $\dimension + 1$ among these regions have a non-empty intersection, it holds that $\bigcap_{i=1}^kX_i\neq\emptyset$. Hence, there must be some $\{C_{i_1},\ldots,C_{i_{d+1}}\}\subsetneq\{C_1,\ldots,C_k\}$ such that $\embedding\models_\method C_{i_1}\sqcap\ldots\sqcap C_{i_{d+1}}\sqsubseteq\bot$.
\end{remark}

Entailment closure and strong faithfulness together guarantee that an $\method$-model behaves as a canonical model. 

\begin{restatable}{proposition}{kbentailedandstrongkbfaitheq}    
Let \ontology be an $\mathcal{L}$-TBox, \Amc an ABox and $\embedding$ an $\method$-{model} of $\Kmc=\Tmc\cup\Amc$. Then the following holds.
\begin{itemize}
\item If $\embedding$ is TBox-entailed and strongly TBox-faithful then for every TBox axiom $\alpha$ in $\mathcal{L}$, $\new{\Kmc}\models\alpha$ iff $\embedding\models_\method \alpha$.
\item If $\embedding$ is ABox-entailed and strongly ABox-faithful then for every assertion $\alpha$, $\Kmc\models\alpha$ iff  $\embedding\models_\method \alpha$. 
\end{itemize}
\end{restatable}

To study embedding methods \wrt entailment closure and faithfulness, we define the following properties.

\begin{property}[Ability]\label{prop:existence-faithfulness}
Let $\Ymc\in \{\text{TBox, ABox, KB}\}$. 
We say that \method under $S_\method$ is able to be (weakly/strongly) \Ymc-faithful for a language $\mathcal{L}$ if for every satisfiable  
\Lmc-KB $\kb$, there exists an \method-model \embedding of $\Kmc$ such that \embedding interpreted under $S_\method$ is (weakly/strongly) \Ymc-faithful. The \Ymc-entailed ability is defined as expected.
\end{property}

\begin{property}[Guarantee]\label{prop:guarantee-faithfulness}
Let $\Ymc\in \{\text{TBox, ABox, KB}\}$. 
We say that \method under $S_\method$ is guaranteed to be (weakly/strongly)
\Ymc-faithful for a language $\mathcal{L}$ if, for every satisfiable KB $\kb$, \method always produces an \method-model \embedding of $\Kmc$ such that \embedding interpreted under $S_\method$ is (weakly/strongly) \Ymc-faithful.  
The \Ymc-entailed guarantee is defined as expected.
\end{property}

\begin{restatable}{proposition}{guaranteeABoxTBoximpliesguaranteeKB}\label{prop:guaranteeABoxTBoximpliesguaranteeKB}
If $\method$ is guaranteed to be both strongly \new{(resp.\ weakly)}
TBox-faithful and strongly \new{(resp.\ weakly)} ABox-faithful for $\mathcal{L}$ then 
it is guaranteed to be strongly \new{(resp.\ weakly)} KB-faithful for \Lmc. 
The same holds when considering the entailment closure guarantee property. 
\end{restatable}
This does not hold for
the \emph{ability} properties: e.g., the existence of a TBox-entailed \method-model and   
the existence of an ABox-entailed \method-model do not imply the existence of a KB-entailed \method-model. 

\subsection{Expressiveness}

We extend the notion of full expressiveness~\citep{10.5555/3327144.3327341}, a well-known characteristic considered for KG embeddings, to languages that include TBox axioms. 

\begin{property}[Full Expressiveness]\label{prop:expressiveness}
\method under $S_\method$ is
\begin{itemize}
\item \emph{fully TBox-expressive} for \Lmc if for every two \Lmc-TBoxes $\Tmc,\Tmc'$, with \Tmc satisfiable and $\Tmc'$ disjoint from the deductive closure of $\Tmc$, there exists an \method-model $\embedding$ of $\Tmc$ such that $S_\method(\embedding,\mathcal{L})(\alpha)=0$ for all $\alpha\in\Tmc'$;
    \item \emph{fully ABox-expressive}   if for every two ABoxes $\Amc,\Amc'$, with $\Amc'$ being disjoint from $\Amc$, there exists an \method-model $\embedding$ of $\Amc$ such that $S_\method(\embedding,\mathcal{L})(\alpha)=0$ for all $\alpha\in\Amc'$.
\end{itemize}
Full TBox-expressiveness is extended for KBs as expected.
\end{property}

Full ABox-expressiveness coincides with the notion of full expressiveness from the KG embedding literature, and is tightly related to strong ABox-faithfulness. 

\begin{restatable}{proposition}{fullyexrvsaboxfaithful}\label{prop:fully-expressive-vs-model-strongly-ABox-faithful} $\method$ under $S_\method$ is fully ABox-expressive iff for any ABox \Amc there is an \method-model $\embedding$ of \Amc interpreted under $S_\method$  that is strongly ABox-faithful.    
\end{restatable}

In the KG literature, authors often consider the ability of \emph{capturing patterns} from Table~\ref{patterns}~\citep{BoxE,ExpressivE}. 
They distinguish the ability to capture a single pattern  
or to capture jointly several patterns (possibly of different kinds). 
Indeed, some methods are able to produce an embedding that captures a pattern but not multiple ones, even of the same type \citep[Table 1]{BoxE}.  
It follows from the form of the patterns that 
all sets of patterns are satisfiable (if no facts need to be considered).

\begin{definition}[Capturing  patterns (adapted from \citep{BoxE})]
\label{def:capturingpatterns} Let \Lmc be a language of patterns 
and let $\embedding$ be an \method-embedding. 
  $\embedding$ interpreted under $S_\method$
\begin{itemize} \item captures \emph{exactly a pattern} $\pattern\in\Lmc$ if $S_\method(\embedding,\mathcal{L})(\pattern)=1$; \item captures exactly a \emph{set of patterns} $\patternset=\{\pattern_1,\ldots, \pattern_n\}\subseteq\Lmc$ if it captures exactly $\pattern_i$, for all $1\leq i\leq n$; 
\item captures \emph{exclusively} a set of patterns 
$\patternset$ if for every pattern $\pattern$ in $\mathcal{L}$, 
$S_\method(\embedding,\mathcal{L})(\pattern)=1$ only if $\patternset\models \pattern$.
\end{itemize}
\end{definition}

\begin{property}[Ability to capture (adapted from \citep{BoxE})]\label{prop:capture}
We say that \method under $S_\method$  is able to capture (exactly/exclusively) $\mathcal{L}$ if for any finite set of patterns $\patternset$ expressed in $\mathcal{L}$, 
there exists an \methodembedding{\method} interpreted under $S_\method$ that captures (exactly/exclusively) $\patternset$. 
\end{property}

We relate this property with strong faithfulness ability.

\begin{restatable}{proposition}{captexacexclueqtbfaith}\label{prop:capturingfaithfulness} 
If \Lmc is a language of patterns, then $\method$ is able to capture exactly and exclusively \Lmc iff \new{for any finite set of patterns $\patternset$ expressed in $\mathcal{L}$, 
there exists a strongly TBox-faithful $\method$-model of $\patternset$. } 
\end{restatable}

\begin{example}\label{ex:expressive-dont-inject-exclu}
Let $\Lmc$ be the language of exclusion patterns built on $\NR=\{r_1,r_2\}$ (\ie $\Lmc=\{r_1\sqsubseteq \neg r_2, r_2\sqsubseteq \neg r_1\}$).
Theorem 5.1 in \cite{ExpressivE} shows that \ExpressivE is fully ABox-expressive. 
It can be easily checked that \ExpressivE is 
fully TBox-expressive for $\Lmc$ (see \cite[Theorem~5.2]{ExpressivE}). 
However, \ExpressivE is \emph{not} fully KB-expressive for this language because it is not complete for $\Lmc$
(see Example~\ref{ex:convex-not-complete}, which can be instantiated for \ExpressivE since the hyper-parallelograms are convex).\footnote{Though, the authors also consider a weaker semantics where pattern satisfaction is defined w.r.t.~grounded pattern instances.} 
In the same way, Theorem~5.1  in \citep{BoxE}
shows that \BoxE is fully ABox-expressive and 
Theorem~5.3 shows that it is fully TBox-expressive for $\Lmc$ while, 
by \cref{ex:boxeincomplete}, \BoxE is not complete thus \emph{not} fully KB-expressive for $\Lmc$. 
\end{example}

\new{
Some authors also consider 
\emph{knowledge injection}~\citep{benedikt_et_al:DagRep.9.9.1}, 
which 
broadly refers to the task 
of incorporating explicit (pre-defined) knowledge expressed as 
rules or constraints 
into a machine learning model.
This can   be achieved by constraining 
the training, the output, or the model itself with such patterns required to hold in the model. 
\cite{BoxE} establish that their embeddings can be modified so 
as to provably ensure that they satisfy some patterns (among a restricted class of patterns). 
}

\subsection{Relationships Between Properties}

\begin{figure}[t]
    \centering
    \includegraphics[scale=0.25]{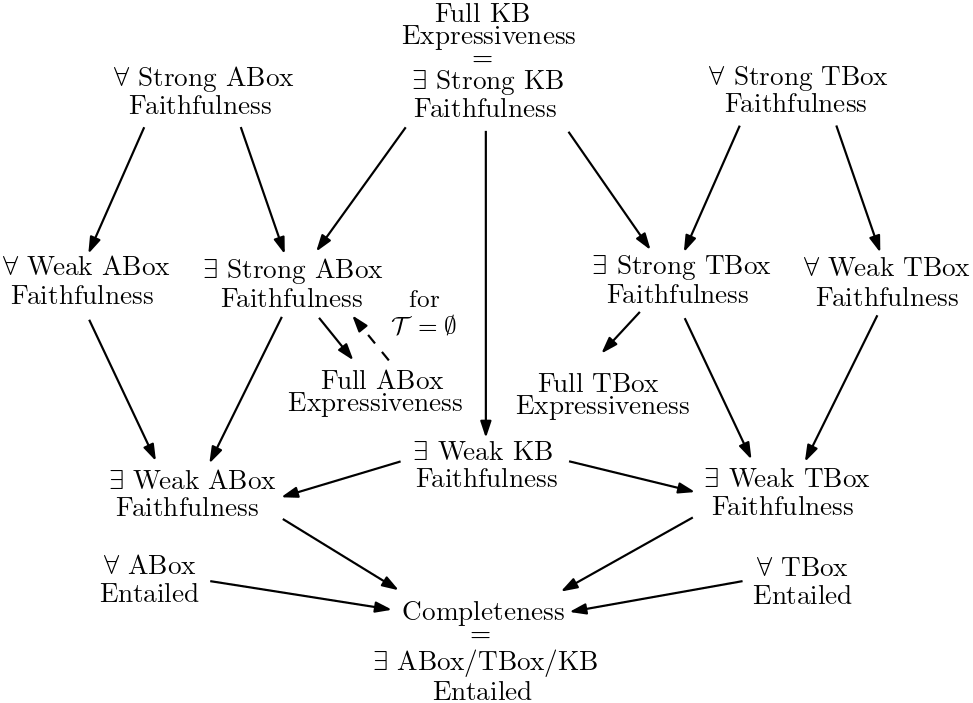}
    \caption{
    Relationships between the properties when the DL language is \emph{finite}. 
    An arrow from X to Y 
    indicates that X implies Y. 
    }
    \label{fig:finlangdiag}    
\end{figure}

\begin{table*}
\centering
{\small
\begin{tabular}{@{}lccccccccc@{}}
\toprule
 Embedding method & Convex & Al-cone &\(\ELEm\)&\(\EmEL\)&\(\ELBE\)&\(\BoxEL\)&\(\BoxTwoEL\)&\(\BoxE\) &\(\ExpressivE\)
\\
TBox language   & $\mathcal{ELHI}_\bot^1$ & $\ALC_{\text{r}}$ & $\mathcal{ELO}_\bot^1$  & $\mathcal{ELHO}(\circ)_\bot^1$& $\mathcal{ELO}_\bot^1$ & $\mathcal{ELO}_\bot^1$ & $\mathcal{ELHO}(\circ)_\bot^1$  & patterns$^2$& patterns
\\ 
\midrule
Soundness                & \checkmark & \checkmark & \xmark & \xmark& \xmark & \checkmark & \xmark & \checkmark & \checkmark
\\
Completeness            & \checkmark & \checkmark & \xmark & \xmark & \xmark & \xmark & \xmark & \xmark & \xmark
\\
\hline
$\forall$ ABox-Entailed  & \checkmark & \checkmark$^\ddag$   & \xmark & \xmark & \xmark & \xmark & \xmark & \xmark & \xmark \Tstrut
\\
$\forall$ TBox-Entailed  & \checkmark & \checkmark$^\ddag$ & \xmark & \xmark & \xmark & \xmark & \xmark & \xmark & \xmark
\\
\hline
$\exists$ Weak ABox-Faithful.  & \checkmark & \checkmark & \xmark & \xmark & \xmark & \xmark & \xmark & \xmark & \xmark \Tstrut 
\\
$\exists$ Weak TBox-Faithful. & \checkmark & \checkmark$^\ddag$ & \xmark & \xmark & \xmark & \xmark & \xmark & \xmark & \xmark\\
$\exists$ Weak KB-Faithful. & \checkmark & \checkmark$^\ddag$ & \xmark & \xmark & \xmark & \xmark & \xmark & \xmark & \xmark \\
$\forall$ Weak ABox-Faithful.  & \checkmark & \checkmark & \xmark & \xmark & \xmark & \xmark & \xmark & \xmark & \xmark                        \\
$\forall$ Weak TBox-Faithful. & \checkmark & ? & \xmark & \xmark & \xmark & \xmark & \xmark & \xmark & \xmark\\
\hline
$\exists$ Strong ABox-Faithful. & \checkmark$^\dag$ & \checkmark  & \xmark & \xmark & \xmark & \xmark & \xmark & \xmark & \xmark \Tstrut
\\
$\exists$ Strong TBox-Faithful. & \checkmark$^\dag$ & \checkmark$^\ddag$  & \xmark & \xmark & \xmark & \xmark & \xmark & \xmark & \xmark
\\
$\exists$ Strong KB-Faithful. & \checkmark$^\dag$ & \checkmark$^\ddag$  & \xmark & \xmark & \xmark & \xmark & \xmark & \xmark & \xmark
\\
$\forall$ Strong ABox-Faithful. & \xmark & \xmark & \xmark & \xmark & \xmark & \xmark & \xmark & \xmark & \xmark
\\
$\forall$ Strong TBox-Faithful. & \xmark  &  \xmark  & \xmark & \xmark & \xmark & \xmark & \xmark & \xmark & \xmark  
\\
\hline
Full ABox Expressiveness & \checkmark &  \checkmark & \xmark   & \xmark & \xmark & \xmark & \checkmark & \checkmark & \checkmark \Tstrut
\\
Full TBox Expressiveness & \checkmark$^\dag$ & \checkmark$^\ddag$ & \xmark & \xmark & \xmark & \xmark &  \xmark &\checkmark  & \checkmark$^{\circ}$
\\ \bottomrule
\end{tabular}
}
\caption{Properties of KB embedding methods. $^1$ in normal form. $^2$ without composition. For $\dag$ cases, results shown for $\mathcal{ELH}$ but we conjecture they also hold for $\mathcal{ELHI}_\bot$.
For $\ddag$ cases, results are for $\ALC_{\text{p}}$.
The $^{\circ}$ result is for the language of \emph{positive} patterns (without negation).  
 Since all the languages considered are finite, $\exists$ ABox-, TBox-, KB-\complete coincide with completeness, and full KB-expressiveness coincides with $\exists$ strong KB-faithfulness (\cf Figure \ref{fig:finlangdiag}). By Proposition \ref{prop:guaranteeABoxTBoximpliesguaranteeKB}, for $\text{X}\in\{$Entailed, Strong Faith., \new{Weak Faith.}$\}$, 
 $\forall$ KB-X holds iff $\forall$ ABox-X and $\forall$ TBox-X hold.}
\label{tab:kbkgmethodProperties}
\end{table*}

We now briefly discuss the relationships between the properties, considering two cases: one for the general case (with possibly \textit{infinite} languages) (\cref{fig:inflangdiag}), and one for the special case of \textit{finite} languages (\cref{fig:finlangdiag}). 
 For readability, we omit $\forall$ KB Entailed, $\forall$ Strong KB Faithfulness \new{and $\forall$ Weak KB Faithfulness} since they are equivalent to the conjunction of the ABox and TBox versions of the properties (Proposition~\ref{prop:guaranteeABoxTBoximpliesguaranteeKB}). 
Note that many properties imply completeness 
because their definitions assume the existence of an \method-model, and that guarantees imply abilities also because the guarantee definition requires the existence of an \method-model. The other relationships are more informative (e.g., strong faithfulness implies weak faithfulness, and strong KB-faithfulness implies full KB-expressiveness).

\begin{restatable}{theorem}{properties}\label{thm:properties}
The relationships between the properties of embedding methods shown in Figure~\ref{fig:inflangdiag} hold.
\end{restatable}

When the language is finite, we observe that some properties coincide. First, being able to be ABox-, TBox- and KB-entailed is equivalent to being complete. Indeed, in this case, the deductive closure of the KB is finite so by completeness, one can find an \method-model that satisfies every consequence of the KB. Second, being able to be strongly KB-faithful coincides with being fully KB-expressive for a similar reason: one can use the deductive closure to obtain a strongly KB-faithful \method-model when \method is fully KB-expressive. 

\begin{restatable}{theorem}{propertiesfinite}\label{thm:propertiesfinite}
For finite languages, the relationships between the properties of embedding methods in Figure~\ref{fig:finlangdiag} hold.
\end{restatable}

\section{Properties of Selected Methods}\label{sec:methods-properties}

Table \ref{tab:kbkgmethodProperties} shows which of the KB embedding methods of Section~\ref{sec:embeddingmethods} satisfy   the properties introduced in Section~\ref{sec:properties}. 
Since we consider KBs in normal form and finite sets $\NI$, $\NC$, and $\NR$, all languages are finite so we only consider properties of Figure~\ref{fig:finlangdiag}. 
This comparison is  {not} intended to be used to claim that some embedding methods are better than others based on the number of properties they satisfy. Our goal here is only to better understand the theoretical properties of these methods. 
Indeed, recall that not all methods apply to the same languages, so they cannot be directly compared. Also, \new{as mentioned earlier,} those that satisfy   more properties, namely, Convex and Al-cone, 
are not implemented. 

Moreover, \new{depending on the use case,}  
some properties may not be desirable. 
In particular, strong ABox-faithful \method-models are unable to predict new plausible facts that are not entailed by the KB. 
\new{Actually, we should often aim for embedding methods that are sound, complete and guaranteed to be strongly TBox-faithful and weakly ABox-faithful. Indeed, $\method$-models that are strongly TBox-faithful and weakly ABox-faithful give formal guarantees that the TBox part from the source KB is respected, while still allowing downstream tasks such as link prediction to be performed based on the data coming from the ABox.  
}

Most methods we consider either do not cover or fail to represent 
role composition, mutual exclusion, and 
axioms of the form \(\exists r.C\sqsubseteq \bot\). 
E.g., \BoxE and \ExpressivE
are not fully KB expressive for languages with mutual exclusion (\cref{ex:expressive-dont-inject-exclu}). \BoxTwoEL and \EmEL fail to represent role composition (see~\cref{ex:not-strongly-faithfuk-model} and   the definition of \EmEL, which implies that if $E\models_{emel} r_1\circ r_2 \sqsubseteq s$ then $E\models_{emel} r_2\circ r_1 \sqsubseteq s$).
Also, \ELEm, \EmEL, and \ELBE
cannot precisely handle the concept inclusion
\(\exists r.C\sqsubseteq \bot\).
These methods approximate it by \(C\sqsubseteq \bot\), which implies \(\exists r.C\sqsubseteq \bot\) but is not equivalent.

\section{\new{Conclusion and Perspectives}}\label{sec:related}

In this work, we  examine  recent  region-based KB embedding methods through the lens of the properties of their geometric-based semantics. 
\new{Our framework provides a common vocabulary and clarifies relationships between properties of KB embeddings (many of them already considered in the literature). It can be used to guide analysis of new methods and facilitate comparisons between existing and future embedding methods. In particular, while several theoretical properties have been considered relevant, such as those related to faithfulness, in practice, there are no implementations that satisfy them. Hence a novel practical embedding method that would satisfy, e.g., soundness and completeness would already offer more theoretical guarantees than existing ones. The main difficulties encountered for fulfilling the properties are related to the ability of representing role disjointness (together with facts!) and the bottom concept. This calls for some research effort since many natural constraints involve these constructs. For example, Wikidata has hundreds of ``conflict with" constraints, which correspond to disjointness axioms between (complex) DL concepts.
}

Recent works in the KG literature focus on query answering, where the task is not only to  rank facts but expressions in a richer query language~\citep{DBLP:conf/nips/HamiltonBZJL18,DBLP:conf/nips/RenL20,DBLP:conf/iclr/RenHL20,DBLP:conf/nips/ZhangWCJW21,DBLP:conf/icml/BaiLLH23}. 
Note that these works 
consider KGs rather than KBs. 
An exception is the work by \cite{DBLP:conf/rulemlrr/ImenesGO23} which targets DL-Lite KBs by performing query rewriting and then querying the ABox embedding. 
We could extend several of our properties to consider queries, for example define weak or strong query-faithfulness for some query language. This would require to extend the embedding method semantics to evaluate queries.

\section*{Acknowledgements}
This work was supported by the ``PHC AURORA'' programme (projects 49698VH and 341302), funded by the French Ministry for Europe and Foreign Affairs, the French Ministry for Higher Education and Research, and the Research Council of Norway, as well as the OLEARN project (316022) funded by the Research Council of Norway.

\bibliographystyle{kr}

\bibliography{refs.bib}

\begin{thebibliography}{}

\bibitem[\protect\citeauthoryear{Abboud \bgroup et al\mbox.\egroup
  }{2020}]{BoxE}
Abboud, R.; Ceylan, {\.I}.~{\.I}.; Lukasiewicz, T.; and Salvatori, T.
\newblock 2020.
\newblock {BoxE}: {A} box embedding model for knowledge base completion.
\newblock In {\em Proceedings of NeurIPS}.

\bibitem[\protect\citeauthoryear{Baader \bgroup et al\mbox.\egroup
  }{2017}]{DlIntro}
Baader, F.; Horrocks, I.; Lutz, C.; and Sattler, U.
\newblock 2017.
\newblock {\em An Introduction to Description Logic}.
\newblock Cambridge University Press, 1st edition.

\bibitem[\protect\citeauthoryear{Baader, Lutz, and
  Brandt}{2008}]{DBLP:conf/owled/BaaderLB08}
Baader, F.; Lutz, C.; and Brandt, S.
\newblock 2008.
\newblock Pushing the {EL} envelope further.
\newblock In {\em Proceedings of the Fourth {OWLED} Workshop on {OWL:}
  Experiences and Directions}.

\bibitem[\protect\citeauthoryear{Bai \bgroup et al\mbox.\egroup
  }{2023}]{DBLP:conf/icml/BaiLLH23}
Bai, Y.; Lv, X.; Li, J.; and Hou, L.
\newblock 2023.
\newblock Answering complex logical queries on knowledge graphs via query
  computation tree optimization.
\newblock In {\em Proceedings of {ICML}}.

\bibitem[\protect\citeauthoryear{Benedikt \bgroup et al\mbox.\egroup
  }{2020}]{benedikt_et_al:DagRep.9.9.1}
Benedikt, M.; Kersting, K.; Kolaitis, P.~G.; and Neider, D.
\newblock 2020.
\newblock {Logic and Learning (Dagstuhl Seminar 19361)}.
\newblock {\em Dagstuhl Reports} 9(9):1--22.

\bibitem[\protect\citeauthoryear{Bourgaux, Ozaki, and
  Pan}{2021}]{DBLP:conf/dlog/BourgauxOP21}
Bourgaux, C.; Ozaki, A.; and Pan, J.~Z.
\newblock 2021.
\newblock Geometric models for (temporally) attributed description logics.
\newblock In {\em Proceedings of {DL}}.

\bibitem[\protect\citeauthoryear{d'Amato, Quatraro, and
  Fanizzi}{2021}]{DBLP:conf/esws/dAmatoQF21}
d'Amato, C.; Quatraro, N.~F.; and Fanizzi, N.
\newblock 2021.
\newblock Injecting background knowledge into embedding models for predictive
  tasks on knowledge graphs.
\newblock In {\em Proceedings of {ESWC}}.

\bibitem[\protect\citeauthoryear{Guti{\'{e}}rrez{-}Basulto and
  Schockaert}{2018}]{Geometric}
Guti{\'{e}}rrez{-}Basulto, V., and Schockaert, S.
\newblock 2018.
\newblock From knowledge graph embedding to ontology embedding? {An} analysis
  of the compatibility between vector space representations and rules.
\newblock In {\em Proceedings of {KR}}.

\bibitem[\protect\citeauthoryear{Hamilton \bgroup et al\mbox.\egroup
  }{2018}]{DBLP:conf/nips/HamiltonBZJL18}
Hamilton, W.~L.; Bajaj, P.; Zitnik, M.; Jurafsky, D.; and Leskovec, J.
\newblock 2018.
\newblock Embedding logical queries on knowledge graphs.
\newblock In {\em Proceedings of Neur{IPS}}.

\bibitem[\protect\citeauthoryear{Helly}{1923}]{Helly1923}
Helly, E.
\newblock 1923.
\newblock Über mengen konvexer körper mit gemeinschaftlichen punkte.
\newblock {\em Jahresbericht der Deutschen Mathematiker-Vereinigung}
  32:175--176.

\bibitem[\protect\citeauthoryear{Imenes, Guimar{\~{a}}es, and
  Ozaki}{2023}]{DBLP:conf/rulemlrr/ImenesGO23}
Imenes, A.; Guimar{\~{a}}es, R.; and Ozaki, A.
\newblock 2023.
\newblock Marrying query rewriting and knowledge graph embeddings.
\newblock In {\em Proceedings of {R}ule{M}{L}+{R}{R}}.

\bibitem[\protect\citeauthoryear{Jackermeier, Chen, and
  Horrocks}{2024}]{DBLP:journals/corr/abs-2301-11118}
Jackermeier, M.; Chen, J.; and Horrocks, I.
\newblock 2024.
\newblock Dual box embeddings for the description logic {EL}\({}^{\mbox{++}}\).
\newblock In {\em Proceedings of {WWW}}.

\bibitem[\protect\citeauthoryear{Kazemi and
  Poole}{2018}]{10.5555/3327144.3327341}
Kazemi, S.~M., and Poole, D.
\newblock 2018.
\newblock Simple embedding for link prediction in knowledge graphs.
\newblock In {\em Proceedings of {N}eur{IPS}}.

\bibitem[\protect\citeauthoryear{Kulmanov \bgroup et al\mbox.\egroup
  }{2019}]{Kulmanov2019}
Kulmanov, M.; Liu-Wei, W.; Yan, Y.; and Hoehndorf, R.
\newblock 2019.
\newblock {EL} embeddings: Geometric construction of models for the description
  logic {EL}++.
\newblock In {\em Proceedings of IJCAI}.

\bibitem[\protect\citeauthoryear{Lacerda, Ozaki, and
  Guimar{\~{a}}es}{2023}]{DBLP:journals/corr/abs-2310-02198}
Lacerda, V.; Ozaki, A.; and Guimar{\~{a}}es, R.
\newblock 2023.
\newblock Strong faithfulness for {ELH} ontology embeddings.
\newblock {\em CoRR} abs/2310.02198.

\bibitem[\protect\citeauthoryear{Li \bgroup et al\mbox.\egroup
  }{2022}]{app122010690}
Li, W.; Zheng, X.; Gao, H.; Ji, Q.; and Qi, G.
\newblock 2022.
\newblock Cosine-based embedding for completing lightweight schematic knowledge
  in {DL-Lite}core.
\newblock {\em Applied Sciences} 12(20).

\bibitem[\protect\citeauthoryear{Mondal, Bhatia, and Mutharaju}{2021}]{EmEL}
Mondal, S.; Bhatia, S.; and Mutharaju, R.
\newblock 2021.
\newblock {EmEL++}: Embeddings for {EL++} description logic.
\newblock In {\em Proceedings of {AAAI-MAKE}}.

\bibitem[\protect\citeauthoryear{{\"{O}}z{\c{c}}ep, Leemhuis, and
  Wolter}{2020}]{Cone1}
{\"{O}}z{\c{c}}ep, {\"{O}}.~L.; Leemhuis, M.; and Wolter, D.
\newblock 2020.
\newblock Cone semantics for logics with negation.
\newblock In {\em Proceedings of {IJCAI}}.

\bibitem[\protect\citeauthoryear{Pavlovic and Sallinger}{2023}]{ExpressivE}
Pavlovic, A., and Sallinger, E.
\newblock 2023.
\newblock Expressiv{E}: {A} spatio-functional embedding for knowledge graph
  completion.
\newblock In {\em Proceedings of {ICLR}}.

\bibitem[\protect\citeauthoryear{Peng \bgroup et al\mbox.\egroup
  }{2022}]{DBLP:journals/corr/abs-2202-14018}
Peng, X.; Tang, Z.; Kulmanov, M.; Niu, K.; and Hoehndorf, R.
\newblock 2022.
\newblock Description logic {EL++} embeddings with intersectional closure.
\newblock {\em CoRR} abs/2202.14018.

\bibitem[\protect\citeauthoryear{Ren and
  Leskovec}{2020}]{DBLP:conf/nips/RenL20}
Ren, H., and Leskovec, J.
\newblock 2020.
\newblock Beta embeddings for multi-hop logical reasoning in knowledge graphs.
\newblock In {\em Proceedings of Neur{IPS}}.

\bibitem[\protect\citeauthoryear{Ren, Hu, and
  Leskovec}{2020}]{DBLP:conf/iclr/RenHL20}
Ren, H.; Hu, W.; and Leskovec, J.
\newblock 2020.
\newblock Query2box: Reasoning over knowledge graphs in vector space using box
  embeddings.
\newblock In {\em Proceedings of {ICLR}}.

\bibitem[\protect\citeauthoryear{Xiong \bgroup et al\mbox.\egroup
  }{2022}]{BoxEL}
Xiong, B.; Potyka, N.; Tran, T.-K.; Nayyeri, M.; and Staab, S.
\newblock 2022.
\newblock Faithful embeddings for {EL++} knowledge bases.
\newblock In {\em Proceedings of {ISWC}}.

\bibitem[\protect\citeauthoryear{Xiong \bgroup et al\mbox.\egroup
  }{2023}]{DBLP:journals/corr/abs-2304-11949}
Xiong, B.; Nayyeri, M.; Jin, M.; He, Y.; Cochez, M.; Pan, S.; and Staab, S.
\newblock 2023.
\newblock Geometric relational embeddings: {A} survey.
\newblock {\em CoRR} abs/2304.11949.

\bibitem[\protect\citeauthoryear{Zhang \bgroup et al\mbox.\egroup
  }{2021}]{DBLP:conf/nips/ZhangWCJW21}
Zhang, Z.; Wang, J.; Chen, J.; Ji, S.; and Wu, F.
\newblock 2021.
\newblock Cone: Cone embeddings for multi-hop reasoning over knowledge graphs.
\newblock In {\em Proceedings of Neur{IPS}}.

\end{thebibliography}

\appendix

\section{Proofs for Section \ref{sec:properties}}
\kbentailedandstrongkbfaitheq*
 \begin{proof}
 Suppose that $E$ is a TBox-entailed and strongly TBox-faithful $M$-model of $\Kmc=\Tmc\cup\Amc$, and $\alpha$ is a TBox axiom. Then $\new{\Kmc}\models\alpha$ implies that $S_M(E,\Lmc)(\alpha)=1$ as $E$ is TBox-entailed. Conversely, $S_M(E,\Lmc)(\alpha)=1$ implies that $\new{\Kmc}\models\alpha$ by strong TBox-faithfulness.

 Similarly, if $E$ is an ABox-entailed and strongly ABox-faithful $M$-model of $\Kmc$, and $\alpha$ is an assertion, then $\Kmc\models\alpha$ implies $S_M(E,\Lmc)(\alpha)=1$ as $E$ is ABox-entailed and $S_M(E,\Lmc)(\alpha)=1$ implies that $\Kmc\models\alpha$ by strong ABox-faithfulness.
 \end{proof}

\guaranteeABoxTBoximpliesguaranteeKB*
\begin{proof}
Let $M$ be an embedding method that is guaranteed to be both strongly \new{(resp.\ weakly)} ABox-faithful and strongly \new{(resp.\ weakly)} TBox-faithful. Let $\Kmc$ be a satisfiable $\Lmc$-KB and $E$ an $M$-model of $\Kmc$. By assumption, $E$ is both strongly \new{(resp.\ weakly)} ABox-faithful and strongly \new{(resp.\ weakly)} TBox-faithful \new{so by Definition~\ref{def:strongfaithful} \new{(resp.\ Definition~\ref{def:weakfaithful})}, $E$ is strongly \new{(resp.\ weakly)} KB-faithful}.

Suppose that $M$ is both guaranteed to be ABox-entailed and TBox-entailed.  
Let $\Kmc$ be a satisfiable $\Lmc$-KB and $E$ an $M$-model of $\Kmc$. 
By assumption, $E$ is both TBox-entailed and ABox-entailed \new{so by Definition~\ref{def:xmodelcompleteness}, $E$ is KB-entailed. }
\end{proof}

\fullyexrvsaboxfaithful*
\begin{proof}
Let $\mathcal{L}$ be the language of assertions. 
Assume for any \Amc there is an \method-embedding  $\embedding$  interpreted under $S_\method$ that is a 
strongly ABox-faithful \method-model of \Amc. 
Let $\Amc_T$ and $\Amc_F$ be two disjoint sets of true and false facts respectively and let $\embedding$ be a strongly ABox-faithful \method-model of $\Amc_T$. 
Since $\embedding$ is an \method-model, if $\alpha\in\Amc_T$ then
$S_\method(E,\mathcal{L})(\alpha)=1 $.
By strong ABox-faithfulness, if $S_\method(E,\mathcal{L})(\alpha)=1 $ then $\alpha$ is entailed by $\Amc_T$, i.e., $\alpha\in\Amc_T$. This means that $S_\method(E,\mathcal{L})(\alpha)=1 $ iff $\alpha\in\Amc_T$. 
Thus, for every true fact $\alpha\in\Amc_T$, $S_\method(\embedding,\mathcal{L})(\alpha)=1$ and for every false fact $\alpha\in\Amc_F$, $S_\method(\embedding,\mathcal{L})(\alpha)=0$. 
It follows that $\method$ under $S_\method$ is fully ABox-expressive. 

Conversely assume $\method$ under $S_\method$ is fully ABox-expressive and let $\Amc$ be an ABox. 
Consider $\Amc$ as the set of true facts and $\Lmc\setminus\Amc$ (recall that $\Lmc$ is the set of all assertions, which is finite since $\NC$, $\NR$ and $\NI$ are finite) as the set of false facts. 
There exists an  \method-embedding $\embedding$ such that for every true fact $\alpha$, $S_\method(\embedding,\mathcal{L})(\alpha)=1$ and for every false fact $\alpha$, $S_\method(\embedding,\mathcal{L})(\alpha)=0$.
Hence $\alpha\in\Amc$ iff $S_\method(E,\mathcal{L})(\alpha)=1 $. 
It follows that $E$ is an \method-model of $\Amc$ that is strongly ABox-faithful (since there is no TBox). 
\end{proof}

\captexacexclueqtbfaith*
\begin{proof}
Let $\Smc$ be a set of patterns 
expressed in \Lmc 
and let $E$ be an $M$-embedding. Then, $E$ captures $\Smc$ exactly and exclusively iff (i) $E\models_M\alpha$ for all $\alpha\in\Smc$ and (ii) $E\models_M \beta$ implies $\Smc\models\beta$ for all $\beta\in\Lmc$. But (i) says that $E$ is an $M$-model of $\Smc$ and (ii) says that $E$ is a strongly TBox-faithful $M$-model of $\Smc$ for $\Lmc$.
\end{proof}

\subsection{Proofs of Theorems \ref{thm:properties} and \ref{thm:propertiesfinite}}

\begin{restatable}{proposition}{weakstrongfaith}\label{prop:weakstrong}
Let \(\kb\) be a satisfiable KB in \(\llang\) and \(E\) an \method-model of \(\kb\). For $\Ymc\in \{\text{TBox, ABox, KB}\}$,  
if \(E\) is strongly $\Ymc$-faithful, then \(E\) is also weakly $\Ymc$-faithful.
\end{restatable}

\begin{proof}
Let \(\kb = \tbox \cup \abox\) be a satisfiable KB in \(\llang\) and let \(E\) be an \method-model of \(\kb\). 

If \(E\) is strongly ABox-faithful (resp.\ TBox-faithful), then for every assertion \(\alpha\) (resp.\ TBox axiom in \(\llang\)), it holds that $\embedding\models_{\method} \alpha$ implies \(\kb \models \alpha\). 
Since \(\kb \models \alpha\), every model of \(\kb\) satisfies \(\alpha\), hence since $\kb$ is satisfiable, \(\kb \cup \{\alpha\}\) is satisfiable. 
Hence, $\embedding\models_\method \alpha$ implies that  \(\alpha\) is consistent with \(\kb\), i.e., \(E\) is weakly ABox-faithful (resp.\ TBox-faithful). 

\new{If \(E\) is strongly KB-faithful, then by Definition \ref{def:strongfaithful} it is both ABox- and TBox-faithful. We have shown that this implies that $E$ is weakly ABox- and TBox-faithful, so by Definition~\ref{def:weakfaithful}, \(E\) is weakly KB-faithful. }
\end{proof}

\begin{restatable}{proposition}{fullkbexpimpliesexistancestrongkbfaith}\label{prop:strongfaithfulnessimpliesexpressiveness}
We have that $M$ is:
\begin{itemize}
    \item fully TBox-expressive for $\Lmc$ if it is able to be strongly TBox-faithful; 
    \item fully KB-expressive for $\Lmc$ if it is able to be strongly KB-faithful.
\end{itemize}
\end{restatable}
\begin{proof}
For the first item, let $\Tmc,\Tmc'$ be two $\Lmc$-TBoxes such that \Tmc is satisfiable and $\Tmc'$ is disjoint from the deductive closure of $\Tmc$. By assumption there is some strongly TBox-faithful $M$-model $E$ of $\Tmc$. It therefore holds that $S_M(E,\Lmc)(\alpha)=0$, for every $\alpha\in\Tmc'$, because $\Tmc\not\models\alpha$.

For the second item, let $\Kmc,\Kmc'$ be two $\Lmc$-KBs such that $\Kmc'$ is disjoint from the deductive closure of $\Kmc$. By assumption there is some strongly KB-faithful $M$-model $E$ of $\Tmc$ \new{(which is thus strongly ABox- and TBox-faithful)}. It therefore holds that $S_M(E,\Lmc)(\alpha)=0$ for every $\alpha\in\Kmc'$, because $\Kmc\not\models\alpha$ \new{(by strong ABox-faithfulness of $E$ if $\alpha$ is an assertion and by strong TBox-faithfulness of $E$ if it is a TBox axiom)}.
\end{proof}

\properties*
\begin{proof}
All arrows going from a guarantee ($\forall$) to the corresponding existence ($\exists$) hold directly from the definitions of Properties~\ref{prop:existence-faithfulness} and \ref{prop:guarantee-faithfulness}. 
Moreover, the arrows from strong $\Ymc$-faithfulness to weak $\Ymc$-faithfulness are by
\cref{prop:weakstrong}. 
The arrows from ability to be strongly $\Ymc$-faithful to full $\Ymc$-expressiveness are by Propositions~\ref{prop:fully-expressive-vs-model-strongly-ABox-faithful} and \ref{prop:strongfaithfulnessimpliesexpressiveness}. 
\new{We have arrows from KB- to TBox- and ABox- entailment closure and weak and strong faithfulness properties by Definitions \ref{def:xmodelcompleteness}, \ref{def:weakfaithful} and \ref{def:strongfaithful} respectively.}

\new{We remark that we also have arrows from KB- to TBox- and ABox-full expressiveness properties. Indeed, if $\method$ under $S_\method$ is fully KB-expressive, then for every two $\Lmc$-KBs $\Kmc$ and $\Kmc'$ with $\Kmc$ satisfiable and $\Kmc'$ disjoint from the deductive closure of $\Kmc$, there exists an $\method$-model $E$ of $\Kmc$ such that $S_\method(E,\Lmc)(\alpha)=0$ for all $\alpha\in\Kmc'$. By applying this definition on pairs $\Kmc,\Kmc'$ of TBoxes or ABoxes, we obtain that $\method$ is fully TBox- and ABox-expressive. }

Completeness requires that when the KB is satisfiable there is an \method-model for the embedding method \method (Property~\ref{prop:completeness}). Completeness is thus a consequence of the various properties of 
ability to be weakly or strongly ABox/TBox/KB-faithful or -entailed, by the way we have defined Property \ref{prop:existence-faithfulness}. One way to think about completeness is as a basic requirement needed for the entailed and faithfulness notions, which refer to further properties of an embedding model. 
Regarding the arrow from full KB-expressiveness to completeness, this comes from the fact that Property \ref{prop:expressiveness} requires the existence of an \method-model of $\Tmc\cup\Amc$ for full KB-expressiveness (while it only requires a model of $\Tmc$ or $\Amc$ when concerned with TBox- or ABox-expressiveness).

Finally, the dashed arrow that indicates that full ABox-expressiveness implies ability to be strongly ABox-faithful \emph{in the case where the TBox $\Tmc$ is empty} comes from Proposition~\ref{prop:fully-expressive-vs-model-strongly-ABox-faithful}. 
\end{proof}

\propertiesfinite*

\begin{proof}
All the arrows from Figure~\ref{fig:inflangdiag} also hold in Figure~\ref{fig:finlangdiag} and hold by Theorem~\ref{thm:properties}. It remains to show that (1) full KB expressiveness implies $\exists$ strong KB faithfulness, and (2) that completeness implies all of $\exists$ KB/TBox/ABox-entailed notions to show the equivalences between properties shown in Figure~\ref{fig:finlangdiag} in the case of finite languages. 
Assume that $\Lmc$ is a finite language.

\noindent(1) Assume that $M$ is fully KB-expressive for $\Lmc$ and let $\Kmc=\Tmc\cup\Amc$ be a satisfiable $\Lmc$-KB. Define $\Kmc':=\{\alpha \mid \alpha\in\Lmc, \Kmc\not\models\alpha\}$. Since $\Lmc$ is finite, it follows that $\Kmc'$ is an $\Lmc$-KB. Moreover, $\Kmc'$ is disjoint from the deductive closure of $\Kmc$ by construction. Hence, by full KB-expressiveness, there is an $M$-model $E$ of $\Kmc$ such that $E\not\models_M\alpha$ for any $\alpha\in\Kmc'$. It follows that $E\not\models_M\alpha$ for any $\alpha$ such that $\Kmc\not\models\alpha$ \new{so that $E$ is strongly KB-faithful}. 

\noindent(2) Suppose that $M$ is complete for $\Lmc$ and let $\Kmc=\Tmc\cup\Amc$ be an $\Lmc$-KB. We show that there exists a KB-entailed $M$-model of $\Kmc$ (hence ABox- and TBox-entailed). Let $\Kmc^+$ be the deductive closure of $\Kmc$, i.e. $\Kmc^+:=\{\alpha\;|\;\alpha\in\Lmc,\Kmc\models\alpha\}$. Since $\Lmc$ is finite, $\Kmc^+$ is an $\Lmc$-KB. By completeness, $\Kmc^+$ has an $M$-model $E$. By construction, $E$ is \new{an $M$-model of $\Kmc$ (since $\Kmc\subseteq\Kmc^+$) which is ABox-entailed and TBox-entailed, i.e., $E$ is a} KB-entailed $M$-model of $\Kmc$.
\end{proof}

\section{Results in Table \ref{tab:kbkgmethodProperties}}

Recall that all DL languages considered in this section are finite. This allows us to use Theorem~\ref{thm:propertiesfinite} to obtain that some properties hold or does not hold based on the status of other properties.

\subsection{Properties of Convex Geometric Models}\label{subsec:propconvex}

\begin{proposition}[Soundness and completeness]\label{convex-sound-complete}
For $\mathcal{ELHI}_\bot$ in normal form, the embedding method based on convex geometric models under the semantics given by $\models_{\mi{conv}}$ is sound and complete. 
\end{proposition}
\begin{proof}
This directly follows from Corollary 1 in \citep{Geometric}.
\end{proof}

In this section, given a convex geometric model $\embedding$ of an $\mathcal{ELHI}_\bot$ TBox $\Tmc$ in normal form and an ABox $\Amc$, we define the classical model $\Imc_\embedding$ of $\Tmc$ and $\Amc$ (see \citep{DBLP:conf/dlog/BourgauxOP21}): $\Delta^{\Imc_\embedding}=\mathbb{R}^\dimension$, $a^{\Imc_\embedding}=\embedding(a)$ for every $a\in\NI$, $A^{\Imc_\embedding}=\embedding(A)$ for every $A\in\NC$, and $r^{\Imc_\embedding}=\{(u,v)\mid u\oplus v\in \embedding(r)\}$ for every $r\in\NR$. 

It is easy to check that $\Imc_\embedding$ is indeed a classical model of $\Tmc\cup\Amc$ by the definition of $\embedding\models_{\mi{conv}} \Tmc\cup\Amc$. For every $\mathcal{ELHI}_\bot$ concept inclusion (in normal form) or role inclusion $X\sqsubseteq Y\in\Tmc$, $\embedding\models_{\mi{conv}} X\sqsubseteq Y$ means that $\embedding(X)\subseteq \embedding(Y)$ and this implies $X^{\Imc_\embedding}\subseteq Y^{\Imc_\embedding}$. Similarly, for every concept assertion $A(a)$, $\embedding\models_{\mi{conv}} A(a)$ means that $\embedding(a)\in \embedding(A)$, \ie $a^{\Imc_\embedding}\in A^{\Imc_\embedding}$, and for every role assertion $\embedding\models_{\mi{conv}} r(a,b)$ means that $\embedding(a)\oplus\embedding(b)\in \embedding(r)$, \ie $(a^{\Imc_\embedding},b^{\Imc_\embedding})\in r^{\Imc_\embedding}$. 

\begin{proposition}[Entailment closure]\label{convex-entailed}
For $\mathcal{ELHI}_\bot$ in normal form, any embedding method based on convex geometric models under the semantics given by $\models_{\mi{conv}}$ is guaranteed to be KB-entailed. 

Hence (by Theorem~\ref{thm:propertiesfinite}) it is also guaranteed to be TBox- and ABox-entailed, and able to be KB-, TBox- and ABox-entailed. 
\end{proposition}
\begin{proof}
Let $\Tmc$ be an  $\mathcal{ELHI}_\bot$ TBox in normal form and let $\Amc$ be an ABox such that $\Kmc=\Tmc\cup\Amc$ is satisfiable. 
Let $\embedding$ be a convex geometric model of $\Kmc$.

For every $\mathcal{ELHI}_\bot$ axiom or assertion $\alpha$, if $\embedding\not\models_{\mi{conv}} \alpha$, it is easy to check that $\Imc_\embedding\not\models \alpha$, so that $\alpha$ is not entailed by $\Kmc$. Hence, $\embedding$ interpreted under the semantics given by $\models_{\mi{conv}}$ is TBox- and ABox-entailed\new{, i.e., KB-entailed.} 

Since this is true for any $\embedding$, the embedding method is guaranteed to be KB-entailed.
\end{proof}

\begin{proposition}[Weak faithfulness]\label{convex-weak-faithfull}
For $\mathcal{ELHI}_\bot$ in normal form, any embedding method based on convex geometric models under the semantics given by $\models_{\mi{conv}}$ is guaranteed to be weakly KB-faithful. 

Hence (by Theorem~\ref{thm:propertiesfinite}) it is also guaranteed to be weakly TBox- and ABox-faithful, and able to be weakly KB-, TBox- and ABox-faithful.
\end{proposition}
\begin{proof}
Let $\Tmc$ be an  $\mathcal{ELHI}_\bot$ TBox in normal form and let $\Amc$ be an ABox such that $\Kmc=\Tmc\cup\Amc$ is satisfiable. 
Let $\embedding$ be a convex geometric model of $\Kmc$. 

For every $\mathcal{ELHI}_\bot$ concept inclusion (in normal form) or role inclusion, or assertion $\alpha$, we have seen that $\embedding\models_{\mi{conv}} \alpha$ 
implies $\Imc_\embedding\models \alpha$. 
Since $\Imc_\embedding$ is a classical model of $\Tmc\cup\Amc$, it follows that $\alpha$ is consistent with $\Tmc\cup\Amc$. Hence $\embedding$ interpreted under the semantics given by $\models_{\mi{conv}}$ is weakly KB-faithful.  
 
It follows that every convex geometric model of $\Kmc$ is weakly KB-faithful, so the embedding method is guaranteed to be weakly KB-faithful.
\end{proof}

\begin{proposition}[Absence of strong faithfulness guarantee]\label{convex-not-guarantee-strong-faithfull}
An embedding method based on convex geometric models under the semantics given by $\models_{\mi{conv}}$ is not guaranteed to be strongly TBox- nor ABox-faithful, for every language that can express concept hierarchy. 
\end{proposition}
\begin{proof}
Consider $\Tmc=\{A\sqsubseteq B\}$, $\Amc=\{B(a)\}$ and the convex geometric model of $\Amc$ and $\Tmc$ defined by $\embedding(A)=\embedding(B)=\mathbb{R}$ and $\embedding(a)=(0)$.   $\embedding\models_{\mi{conv}} B\sqsubseteq A$ while $\Tmc\cup\Amc\not\models B\sqsubseteq A$ so $\embedding$ interpreted under the semantics given by $\models_{\mi{conv}}$ is not strongly TBox-faithful, and $\embedding\models_{\mi{conv}} A(a)$ while $\Tmc\cup\Amc\not\models A(a)$ so $\embedding$ interpreted under the semantics given by $\models_{\mi{conv}}$ is not strongly ABox-faithful. 
\end{proof}

\begin{proposition}[Strong faithfulness ability and full expressiveness]\label{convex-strongly-KB-faithful}
For $\mathcal{ELH}$ in normal form, any embedding method based on convex geometric models under the semantics given by $\models_{\mi{conv}}$ is able to be strongly KB-faithful.

Hence (by Theorem~\ref{thm:propertiesfinite}) it is also able to be strongly TBox-faithful and ABox-faithful, and it is fully KB-, TBox- and ABox-expressive. Moreover, full ABox-expressiveness is independent from the DL language considered. 
\end{proposition}
\begin{proof}
The existence of a strongly TBox- and ABox-faithful, hence strongly KB-faithful, convex geometric model for $\mathcal{ELH}$ KBs in normal form has been shown by \cite{DBLP:journals/corr/abs-2310-02198}. 
\end{proof}

\subsection{Properties of Al-Cone Models}

\begin{proposition}[Soundness and completeness]\label{thm:soundcone}
    For $\ALC_p$ and $\ALC_r$, any embedding method based on Al-cone models under the semantics given by $\models_{\mi{cone}}$ is sound and complete.
\end{proposition}

\begin{proof}
    The results for $\ALC_p$ and $\ALC_r$ come  from 
    Propositions 2 and 5 in \citep{Cone1}.
\end{proof}

Given an  Al-cone-model $E$ of a  KB $\Kmc=\Tmc\cup\Amc$, 
we define a classical interpretation
$\Imc_E$ of \Kmc
with $\Delta^{\Imc_E}=\mathbb{R}^\dimension$,
 $a^{\Imc_E}=E(a)$ for every $a\in\NI$,
 $A^{\Imc_E}=E(A)$ for every $A\in\NC$, and
$r^{\Imc_E}=E(r)$ for every $r\in\NR$.  
In contrast with our construction for geometric models in \cref{subsec:propconvex}, the interpretation $\Imc_E$ may \emph{not} be a model of \Kmc. This is because 
the $\ALC$ fragments considered can express full negation and, consequently, disjunction. So e.g. it can be that  $E(a)\in E(A\sqcup B)$
but $E(a)\notin E(A)$ and $E(a)\notin E(B)$. Such embedding would represent a KB \Kmc such that $\Kmc\models (A\sqcup B)(a)$ but 
$\Kmc\not\models A(a)$ and
$\Kmc\not\models B(a)$. 

 In $\ALC_p$, one can obtain  classical models of \Kmc by modifying $E$ so as to move individual names away from regions with only partial information to regions with full information (c.f. proof of Proposition~2 by \cite{Cone1}). This would yield multiple
 modified versions of $\embedding$.
In our example, in one version   of $E$ we could have $E(a)\in E(A)$ and in another one $E(a)\in E(B)$. 
Let
 $\Imf_E =\{\Imc^1_E,\Imc^2_E,\ldots\}$ be the set of  classical models of \Kmc that result from these modified versions $\embedding^i$ of  $E$. 
 Such $\Imc^i_E$ interpretations are indeed classical models of \Kmc.
 Similar to the argument  in \cref{subsec:propconvex}, for a modified version $\embedding^i$ of $\embedding$, we now have that 
 $\embedding^i\models_{cone} C\sqsubseteq D$ means $\embedding^i(C)\subseteq \embedding^i(D)$ and this implies 
 $C^{\Imc^i_E}\sqsubseteq D^{\Imc^i_E}$, for every concept inclusion $C\sqsubseteq D$. 
 Moreover,  for every concept assertion $A(a)$, $\embedding^i \models_{cone} A(a)$  means $\embedding^i(a)\in \embedding^i(A)$ and this implies $a^{\Imc^i_E}\in A^{\Imc^i_E}$,  and for every role assertion $r(a,b)$, $\embedding^i \models_{cone} r(a,b)$ means $(\embedding^i(a),\embedding^i(b))\in \embedding^i(r)$ and this implies $(a^{\Imc^i_E},b^{\Imc^i_E})\in r^{\Imc^i_E}$. 
 
\begin{proposition}[Entailment closure]\label{cone-entailed}
For $\mathcal{ALC}_{p}$, any embedding method based on Al-cone   models under the semantics given by $\models_{\mi{cone}}$ is guaranteed to be KB-entailed. 

Hence (by Theorem~\ref{thm:propertiesfinite}) it is also guaranteed to be TBox- and ABox-entailed, and able to be KB-, TBox- and ABox-entailed. 
\end{proposition}
\begin{proof}
Let $\Tmc$ be an  $\mathcal{ALC}_p$ TBox in normal form and let $\Amc$ be an ABox such that $\Kmc=\Tmc\cup\Amc$ is satisfiable. 
Let $\embedding$ be an Al-cone   model of $\Kmc$.

For every $\mathcal{ALC}_p$ axiom or assertion $\alpha$, if $\embedding\not\models_{\mi{cone}} \alpha$, there is $\Imc^i_\embedding\in\Imf_E$
such that  $\Imc^i_\embedding\not\models \alpha$, so that $\alpha$ is not entailed by $\Kmc$. Hence, $\embedding$ interpreted under the semantics given by $\models_{\mi{cone}}$ is TBox- and ABox-entailed\new{, i.e., KB-entailed.}  
We state this for $\mathcal{ALC}_p$ since the claim by the authors that one can obtain classical models by modifying $E$ is made in the proof of  Proposition~2 by \cite{Cone1}, which is for $\mathcal{ALC}_p$.

Since $\embedding$ was an arbitrary embedding model, the embedding method is guaranteed to be KB-entailed.
\end{proof}
\begin{proposition}[Ability to be faithful and fully expressive for $\ALC_p$]
    For $\ALC_p$,  any method based on Al-cone models under the semantics given by $\models_{\mi{cone}}$ is able to be   strongly KB-faithful.
    
    Hence (by Theorem~\ref{thm:propertiesfinite}) it is also able to weakly KB-faithful and to be strongly and weakly TBox-faithful and ABox-faithful, and it is fully KB-, TBox- and ABox-expressive. Moreover, full ABox-expressiveness is independent from the DL language considered.
\end{proposition}
\begin{proof}
    The result comes from Proposition 3 in \citep{Cone1} and the geometric-based semantics of Al-cone.
\end{proof}

\begin{proposition}[Weak ABox-faithfulness for $\ALC_p$ and $\ALC_r$]\label{thm:coneweakguarantee}
    For $\ALC_p$ and $\ALC_r$, any method based on Al-cone  under the semantics given by $\models_{\mi{cone}}$ is guaranteed to be weakly ABox-faithful.

    Hence (by Theorem~\ref{thm:propertiesfinite}) it is also able to be weakly ABox-faithful.
\end{proposition}

\begin{proof}
Let $\Tmc$ be an   $\ALC_r$ TBox   and let $\Amc$ be an ABox such that $\Kmc=\Tmc\cup\Amc$ is satisfiable. 
Let $\embedding$ be an Al-cone model of $\Kmc$. 
We first argue that if $\embedding\models_{cone} A(a)$ 
then $A(a)$  is consistent with 
\Kmc. If $\embedding\models_{cone} A(a)$ then $\embedding(a)\in \embedding(A)$.
It follows from the proof of Proposition~5 in \citep{Cone1}, that one can extend 
$\Imc_E$ to be a classical model of \Kmc. 
Let $\Imc^i_E$ be the result of such extension of $\Imc_E$. 
Since   $\Imc^i_E$ is a classical model of \Kmc and $a^{\Imc^i_E}\in A^{\Imc^i_E}$ we have that 
$A(a)$ is consistent with \Kmc.
A similar argument can be given for role assertions. 
This argument does not apply to TBox axioms because it can be that
$\embedding\models_{cone} C\sqsubseteq D$ but $\Imc^i_E\not\models C\sqsubseteq D$ (e.g. since it extends $\Imc_E$, it can have an element in $C^{\Imc^i_E}$ but not in $D^{\Imc^i_E}$). 
\end{proof}

\begin{proposition}[Strong ABox-faithfulness ability for $\ALC_p$ and $\ALC_r$]
    For $\ALC_p$ and $\ALC_r$, any method based on Al-cone models under the semantics given by $\models_{\mi{cone}}$ is able to be 
    strongly ABox-faithful.    
\end{proposition}

\begin{proof}
The result comes from Proposition 5 in \citep{Cone1} and the geometric-based semantics of Al-cone.
\end{proof}

\begin{proposition}[Absence of strong faithfulness guarantee]\label{cone-not-guarantee-strong-faithfull}
An embedding method based on Al-cone models under the semantics given by $\models_{\mi{cone}}$ is not guaranteed to be strongly TBox- nor ABox-faithful, for every language that can express concept hierarchy. 
\end{proposition}
\begin{proof}
    This can be proved as in \cref{convex-not-guarantee-strong-faithfull}. One can consider $\Tmc=\{A\sqsubseteq B\}$, $\Amc=\{B(a)\}$ and a $d$-dimensional Al-cone model that assigns the same Al-cone to both $A$ and $B$ and $\embedding(a)$ to some $\mathbb{R}^d\setminus \vec{0}$ in $\embedding(B)$.
Since, by definition, $\embedding(A)=\embedding(B)$, we have that $\embedding\models_{cone} B\sqsubseteq A$ and $\embedding \models_{cone} A(a)$, which are not logical consequences of $\Tmc\cup\Amc$.
\end{proof}

\subsection{Properties of \ELEm, \EmEL, and \ELBE}

We group 
\ELEm, \EmEL, and \ELBE in this section since they have many common properties.

\begin{proposition}[Absence of soundness]
For $\mathcal{ELO}_\bot$ (resp.\ $\mathcal{ELHO}(\circ)_\bot$) in normal form, 
\ELEm and \ELBE (resp.\ \EmEL) are not sound.
\end{proposition}
\begin{proof}
See Example \ref{ex:elem-not-sound}, which uses the fact that the axioms $A\sqsubseteq\bot$ and $A(a)$ (that is, $\{a\}\sqsubseteq A$) form an unsatisfiable KB. 
\ELEm, \EmEL, and \ELBE 
have embedding models for this KB, so they are not sound.
The issue is that nominals can be mapped to the same structures as concept names. In \ELEm and \EmEL, concepts  and nominals are mapped to open balls. Satisfaction of $A\sqsubseteq\bot$ means $\ball(A)=\emptyset$. One can create an embedding model $\embedding$ for these methods
satisfying $\{a\}\sqsubseteq A$ by
mapping $\{a\}$ to an empty ball.
So, in symbols, $\embedding\models_{elem} \{a\}\sqsubseteq A$ and $\embedding\models_{emel} \{a\}\sqsubseteq A$.
\ELBE maps concept names and nominals to boxes. Satisfaction of $A\sqsubseteq\bot$ means $e_o(A)=\vec{0}$. One can create an $\ELBE$-model $\embedding$ for this KB by mapping
$\{a\}$ to the same box as $\boxe(A)$, so $\embedding\models_{elbe} \{a\}\sqsubseteq A$.
\end{proof}

\begin{proposition}[Absence of completeness, faithfulness, entailment closure, full KB-expressiveness]
For $\mathcal{ELO}_\bot$ (resp.\ $\mathcal{ELHO}(\circ)_\bot$) in normal form, 
\ELEm and \ELBE (resp.\ \EmEL)  are not complete. 

Hence (by Theorem~\ref{thm:propertiesfinite}), they are not able (nor guaranteed) to be weakly or strongly KB-, TBox- or ABox-faithful, they are not able (nor guaranteed) to be KB-, TBox- or ABox-entailed, and they are not fully KB-expressive.
\end{proposition}
\begin{proof}
Consider the satisfiable $\mathcal{EL}$ TBox $\Tmc=\{\exists r. A\sqsubseteq\bot,\top\sqsubseteq A\}$. For all the listed embedding methods above, embedding the first axiom requires that the radius $\rho(A)$ of $A$ must be $0$ (or in the case of \ELBE, the offset $e_o(A)=\vec{0}$). However, it is also true that for all the listed embedding methods above, embedding the second axiom requires that the region associated with $\top$ is contained in the region associated with $A$, hence $\rho(A)>0$ (in the case of \ELBE, $e_o(A)=\vec{\infty}$); a contradiction.
\end{proof}

\begin{proposition}[Absence of full ABox-expressiveness]
For $\mathcal{ELO}_\bot$ (resp.\ $\mathcal{ELHO}(\circ)_\bot$) in normal form, 
\ELEm and \ELBE (resp.\ \EmEL) are not fully ABox Expressive. 
\end{proposition}
 \begin{proof}
Consider the disjoint ABoxes $\Amc=\{r(a,b),r(b,a)\}$ and $\Amc'=\{r(a,a)\}$ of true and false facts. We prove the case of \ELEm but the cases of \EmEL and \ELBE are entirely similar. The assertions in $\Amc$ are converted to the TBox axioms $\{a\}\sqsubseteq\exists r.\{b\}, \{b\}\sqsubseteq\exists r.\{a\}$ and similarly the single assertion in $\Amc'$ is converted to the axiom $\{a\}\sqsubseteq\exists r.\{a\}$.

 Now suppose that $E$ is an \ELEm-model of $\Amc$ of dimension $d$. By the semantics, we get that
 \[\ball(\{a\}){\subseteq}\ball(\exists r.\{b\}){=}\{\vec{x}\in\mathbb{R}^d\;|\;\vec{x}+E(r)\in\ball(\{b\})\}\]
 and similarly 
 \[\ball(\{b\}){\subseteq}\ball(\exists r.\{a\}){=}\{\vec{x}\in\mathbb{R}^d\;|\;\vec{x}+E(r)\in\ball(\{a\})\}\]
 In particular, this means that $\vec{x}\in\ball(\{a\})$ implies $\vec{x}+E(r)+E(r)\in\ball(\{a\})$. If $\ball(\{a\})=\emptyset$ then $\ball(\{a\})\subseteq\ball(\exists r.\{a\})$ so $E\models_{\mi{elem}} r(a,a)$, although $r(a,a)\in\Amc'$. Otherwise, if $\ball(\{a\})\neq \emptyset$, it must be the case that $E(r)=\vec{0}$ so $\ball(\{a\})\subseteq\ball(\exists r.\{a\})=\ball(\{a\})$ and $E\models_{\mi{elem}}r(a,a)$ follows in this case as well, although $r(a,a)\in\Amc'$. 
 \end{proof}

\begin{proposition}[Absence of full TBox-expressiveness]
For $\mathcal{ELO}_\bot$ (resp.\ $\mathcal{ELHO}(\circ)_\bot$) in normal form, 
\ELEm and \ELBE (resp.\ \EmEL) are not fully TBox-expressive.
\end{proposition}
\begin{proof}
Consider the $\mathcal{EL}$-TBoxes $\Tmc=\{\exists r.A\sqsubseteq\bot\}$ and $\Tmc'=\{A\sqsubseteq\bot\}$: $\Tmc$ is satisfiable and $\Tmc'$ is disjoint from the deductive closure of $\Tmc$. By inspecting the semantic definitions of these embedding methods one sees that the axioms $\exists r.A\sqsubseteq\bot$ and $A\sqsubseteq\bot$ have the same truth-conditions and hence for any $M$-embedding $E$ we have that $E\models_M\Tmc$ implies $E\models_M\Tmc'$. 
\end{proof}

\subsection{Properties of \BoxEL}

\begin{proposition}[Soundness \citep{BoxEL}]
For $\mathcal{ELO}_\bot$ in normal form, \BoxEL is sound.
\end{proposition}
\begin{proof}
  The semantics of \BoxEL is based on a standard DL interpretation $\Imc_\embedding$ built from the embedding $\embedding$, so
  that whenever $E$ is a \BoxEL-model of a KB $\Kmc$ then $\Imc_E\models\Kmc$ as well.
\end{proof}

\begin{proposition}[Absence of completeness, faithfulness, entailment closure, full KB-expressiveness]
For $\mathcal{ELO}_\bot$ in normal form, \BoxEL is not complete.

Hence (by Theorem~\ref{thm:propertiesfinite}), it is not able (nor guaranteed) to be weakly or strongly KB-, TBox- or ABox-faithful, it is not able (nor guaranteed) to be KB-, TBox- or ABox-entailed, and it is not fully KB-expressive.
\end{proposition}
\begin{proof}
Consider the ABox \[\Amc=\{r(a,b),r(a,c),B(b),C(c)\}\] together with the TBox $\Tmc=\{B\sqcap C\sqsubseteq\bot\}$ and set $\Kmc=\Tmc\cup\Amc$. Suppose for contradiction that $E$ is a \BoxEL-embedding of $\Kmc$. We have that $E$ gives rise to a geometric interpretation $\Imc_E$ with domain $\Delta^{\Imc_E}$. Hence we have $(a^{\Imc_E},b^{\Imc_E}),(a^{\Imc_E},c^{\Imc_E})\in r^{\Imc_E}$  which means that $T^r(E(a))=E(b)$ and $T^r(E(a))=E(c)$. It follows that $E(b)=E(c)$. However, we also have that $E(b)\in \boxe(B),E(c)\in \boxe(C)$ with $\boxe(B)\cap \boxe(C)=\emptyset$, which implies that $E(b)\ne E(c)$; a contradiction.
\end{proof}

\begin{proposition}[Absence of full ABox-expressiveness]\label{thm:BoxELnotfullyABoxexpr}
 \BoxEL is not fully ABox-expressive.
\end{proposition}
\begin{proof}
Let $E$ be a \BoxEL embedding of the ABox $\Amc=\{r(a,b), r(a,c), A(b)\}$ with associated geometric model $\Imc_E$. We have seen above how the two role assertions imply that $E(b)=E(c)$. But as $E(b)\in \boxe(A)$ it also follows that $E(c)\in \boxe(A)$, i.e. $\Imc_E\models A(c)$, while $A(c)\not\in\Amc$. Thus for $\Amc'=\{A(c)\}$ we have that $\Amc$ and $\Amc'$ are disjoint, but there is no \BoxEL-model separating them.
\end{proof}

\begin{proposition}[Absence of full TBox-expressiveness]\label{thm:BoxELnotfullyTBoxexpr}
For $\mathcal{ELO}_\bot$ in normal form, \BoxEL is not fully TBox-expressive.
\end{proposition}
\begin{proof}
Let $\Tmc=\{A\sqsubseteq \exists r.B, A\sqsubseteq \exists r.C, B\sqcap C\sqsubseteq D\}$ and $\Tmc'=\{A\sqsubseteq \exists r.D\}$. We show that 
any \BoxEL-model  that satisfies
\Tmc also satisfies $\Tmc'$ (even though the axiom in $\Tmc'$ is not in the deductive closure of \Tmc). 
Let $\embedding $ be a \BoxEL-model of \Tmc. 
Since $\Imc_\embedding\models A\sqsubseteq \exists r.B$, we have that $\embedding(A)\subseteq \{x\mid T^r(x)\in\embedding(B)\}$. Similarly, 
$\embedding(A)\subseteq \{x\mid T^r(x)\in\embedding(C)\}$. Hence $\embedding(A)\subseteq \{x\mid T^r(x)\in(\embedding(B)\cap\embedding(C))\}$. 
However, since $\Imc_\embedding\models B\sqcap C\sqsubseteq D$, $B^{\Imc_\embedding}\cap C^{\Imc_\embedding}\subseteq D^{\Imc_\embedding}$, \ie $\embedding(B)\cap\embedding(C)\subseteq \embedding(D)$. 
Hence $\embedding(A)\subseteq \{x\mid T^r(x)\in\embedding(D)\}$, so $\Imc_E\models A\sqsubseteq \exists r.D$, \ie $\embedding\models_{\mi{boxel}} A\sqsubseteq \exists r.D$. 
\end{proof}

\subsection{Properties of \BoxTwoEL}

\begin{proposition}[Absence of soundness]
For $\mathcal{ELHO}(\circ)_\bot$ in normal form, \BoxTwoEL under the semantics given by $\models_{\mi{box2el}}$ is not sound.
\end{proposition}
\begin{proof}
See Example~\ref{ex:Box2ELNotSound}. Recall that Theorem 1 in \citep{DBLP:journals/corr/abs-2301-11118} does not imply that \BoxTwoEL is sound. Indeed, this theorem is shown for models of loss 0 only, which in particular makes all bumps equal to $\vec{0}$, while we consider more general models with potentially non-zero bumps. 
\end{proof}

\begin{proposition}[Absence of completeness, faithfulness, entailment closure, full KB-expressiveness]\label{box2el-not-complete}
For $\mathcal{ELHO}(\circ)_\bot$ in normal form, \BoxTwoEL under the semantics given by $\models_{\mi{box2el}}$ is not complete. 

Hence (by Theorem~\ref{thm:propertiesfinite}), it is not able (nor guaranteed) to be weakly or strongly KB-, TBox- or ABox-faithful, it is not able (nor guaranteed) to be KB-, TBox- or ABox-entailed, and it is not fully KB-expressive.
\end{proposition}
\begin{proof}
Consider $\Tmc=\{r_1\sqsubseteq r_2, \exists r_1.C\sqsubseteq D_1, \exists r_2.C\sqsubseteq D_2, D_1\sqcap D_2\sqsubseteq \bot\}$ and $\Amc=\{r_1(a,b)\}$, which is translated into $\{a\}\sqsubseteq \exists r_1.\{b\}$. It is easy to check that $\Tmc\cup\Amc$ is satisfiable. 

Assume for a contradiction that there exists a \BoxTwoEL-model $\embedding$ for this KB. 
\begin{itemize}
\item Since $\embedding\models_{\mi{box2el}} \{a\}\sqsubseteq \exists r_1.\{b\}$, $\boxe(\{a\})+\bump(\{b\})\subseteq \head(r_1)$ (note that $\boxe(\{a\})\neq\emptyset$: since $\{a\}$ is a nominal, $\boxe(\{a\})$ is a box of volume 0 with $\embedding(a)$ as lower and upper corner). 
\item Since $\embedding\models_{\mi{box2el}} \exists r_1.C\sqsubseteq D_1$, then $\head(r_1)-\bump(C)\subseteq \boxe(D_1)$. 
\item Since $\embedding\models_{\mi{box2el}} \exists r_2.C\sqsubseteq D_2$, then $\head(r_2)-\bump(C)\subseteq \boxe(D_2)$. 
\item Since $\embedding\models_{\mi{box2el}} r_1\sqsubseteq r_2$, then $\head(r_1)\subseteq \head(r_2)$ so $\head(r_1)-\bump(C)\subseteq \head(r_2)-\bump(C)$. 
\item Hence $\head(r_1)-\bump(C)\subseteq \boxe(D_1)\cap \boxe(D_2)$.
\item Since $\embedding\models_{\mi{box2el}} D_1\sqcap D_2\sqsubseteq \bot$, $\boxe(D_1)\cap \boxe(D_2)=\emptyset$, so $\head(r_1)-\bump(C)=\emptyset$. 
\item It follows that $\head(r_1)=\emptyset$, which contradicts $\boxe(\{a\})+\bump(\{b\})\subseteq \head(r_1)$ and $\boxe(\{a\})\neq\emptyset$.\qedhere
\end{itemize}
\end{proof}

\begin{proposition}[Full ABox-expressiveness]
\BoxTwoEL under the semantics given by $\models_{\mi{box2el}}$ is  fully ABox-expressive. 
\end{proposition}
\begin{proof}
Since ABox assertions are embedded by \BoxTwoEL exactly as by \BoxE, and \BoxE is fully ABox-expressive \cite[Theorem 5.1]{BoxE}, so is \BoxTwoEL. 
\end{proof}

\begin{proposition}[Absence of full TBox-expressiveness]\label{box2el-not-full-TBox-expr}
For $\mathcal{ELHO}(\circ)_\bot$ in normal form, \BoxTwoEL under the semantics given by $\models_{\mi{box2el}}$ is not fully TBox-expressive. 
\end{proposition}
\begin{proof}
Consider $\Tmc=\{r_1\circ r_2\sqsubseteq r_3, \exists r_3.C\sqsubseteq D\}$ and $\Tmc'=\{\exists r_1.C\sqsubseteq D\}$: $\Tmc$ is satisfiable and $\Tmc'$ is disjoint from the deductive closure of $\Tmc$. 
Let $\embedding$ be a \BoxTwoEL-model of $\Tmc$. 
Since $\embedding\models_{\mi{box2el}}r_1\circ r_2\sqsubseteq r_3$, then $\head(r_1)\subseteq \head(r_3)$, and since $\embedding\models_{\mi{box2el}} \exists r_3.C\sqsubseteq D$, then $\head(r_3)-\bump(C)\subseteq \boxe(D)$. It follows that $\head(r_1)-\bump(C)\subseteq \boxe(D)$, so $\embedding\models_{\mi{box2el}} \exists r_1.C\sqsubseteq D$.  
It follows that \BoxTwoEL is not fully TBox-expressive. 
\end{proof}

\subsection{Properties of \BoxE}

\begin{proposition}[Soundness]
\BoxE is sound for the language of patterns which is the union of symmetry, inversion,  hierarchy, intersection, mutual exclusion and asymmetry rules.
\end{proposition}
\begin{proof}
Assume that $\Kmc=\Tmc\cup\Amc$ has a \BoxE-model $\embedding$. Let $\Imc_\embedding$ be the interpretation with $\Delta^{\Imc_\embedding}=\mathbb{R}^\dimension$, for every $a\in\NI$, $a^{\Imc_\embedding}=\vec{e}_a$ (where $\embedding(a)=(\vec{e}_a,\vec{b}_a)$), for every $A\in\NC$, $A^{\Imc_\embedding}=A^{(1)}$ (where $\embedding(A)=(A^{(1)})$), and for every $r\in\NR$, 
\begin{align*}
    r^{\Imc_\embedding}=&\{(x,y)\mid x\in r^{(1)}, y\in r^{(2)}\}\\
   \cup & \{(\vec{e}_c,\vec{e}_d)\mid c,d\in\NI, \vec{e}_c+\vec{b}_d\in r^{(1)}, \vec{e}_d+\vec{b}_c\in r^{(2)}\}
\end{align*} (where $\embedding(r)=(r^{(1)}, r^{(2)})$). 

\begin{itemize}
    \item If $A(a)\in\Amc$, then $\embedding\models_{\mi{boxe}}A(a)$ which implies $\vec{e}_a\in A^{(1)}$ so $\Imc_\embedding \models A(a)$.
    \item If $r(c,d)\in\Amc$, then $\embedding\models_{\mi{boxe}}r(c,d)$ which implies $\vec{e}_c+\vec{b}_d\in r^{(1)}$ and $\vec{e}_d+\vec{b}_c\in r^{(2)}$ so $\Imc_\embedding \models r(c,d)$.
    \item If $r_1\equiv r_2^-\in\Tmc$, then $\embedding\models_{\mi{boxe}} r_1\equiv r_2^-$ which implies $r_1^{(1)}=r_2^{(2)}$ and $r_1^{(2)}=r_2^{(1)}$ so $r_1^{\Imc_\embedding}=(r_2^-)^{\Imc_\embedding}$.
    \item If $r_1\sqsubseteq r_2\in\Tmc$, then  $\embedding\models_{\mi{boxe}} r_1\sqsubseteq r_2$ which implies $r_1^{(1)}\subseteq r_2^{(1)}$ and $r_1^{(2)}\subseteq r_2^{(2)}$ so $r_1^{\Imc_\embedding}\subseteq r_2^{\Imc_\embedding}$.
        \item If $r_1\sqcap r_2\sqsubseteq r_3\in\Tmc$, then  $\embedding\models_{\mi{boxe}} r_1\sqcap r_2\sqsubseteq r_3$ which implies $r_1^{(1)}\cap r_2^{(1)}\subseteq r_3^{(1)}$ and $r_1^{(2)}\cap r_2^{(2)}\subseteq r_3^{(2)}$ so $(r_1\sqcap r_2)^{\Imc_\embedding}\subseteq r_3^{\Imc_\embedding}$.
        \item If $r_1\sqsubseteq \neg r_2\in\Tmc$, then  $\embedding\models_{\mi{boxe}} r_1\sqsubseteq \neg r_2$ which implies $r_1^{(1)}\cap r_2^{(1)}=\emptyset$ or $r_1^{(2)}\cap r_2^{(2)}=\emptyset$ so $r_1^{\Imc_\embedding}\cap r_2^{\Imc_\embedding}=\emptyset$.
        \item If $r_1\sqsubseteq \neg r_1^-\in\Tmc$, then  $\embedding\models_{\mi{boxe}} r_1\sqsubseteq \neg r_1^-$ which implies $r_1^{(1)}\cap r_1^{(2)}=\emptyset$ so $r_1^{\Imc_\embedding}\cap (r_1^-)^{\Imc_\embedding}=\emptyset$.
\end{itemize}
Hence $\Imc_\embedding\models \Kmc$.
\end{proof}

\begin{proposition}[Absence of completeness, faithfulness, entailment closure, full KB-expressiveness]
\BoxE is not complete for any DL language $\Lmc$ containing mutual exclusion patterns.

Hence (by Theorem~\ref{thm:propertiesfinite}), it is not able (nor guaranteed) to be weakly or strongly KB-, TBox- or ABox-faithful, it is not able (nor guaranteed) to be KB-, TBox- or ABox-entailed, and it is not fully KB-expressive.
\end{proposition}
\begin{proof}
See Example \ref{ex:boxeincomplete}.
\end{proof}

\begin{proposition}[Full ABox-expressiveness (\cite{BoxE}, Theorem 5.1)]
\BoxE is fully ABox-expressive.
\end{proposition}

\begin{proposition}[Full TBox-expressiveness (\cite{BoxE}, Theorem 5.3)]
\BoxE is fully TBox-expressive for the language of patterns which is the union of symmetry, inversion,  hierarchy, intersection, mutual exclusion and asymmetry rules.
\end{proposition}

\subsection{Properties of \ExpressivE}

\begin{proposition}[Soundness]
\ExpressivE is sound for the language of role assertions and patterns from Table \ref{patterns} (i.e. the language containing all symmetry, inversion,  hierarchy, intersection, composition, mutual exclusion and asymmetry patterns). 
\end{proposition}
\begin{proof}
Let $\Kmc=\Tmc\cup\Amc$ be a KB that has an \ExpressivE-model $E$, i.e. $E\models_{expr}\Kmc$. 
We build an interpretation $\Imc_E$ such that $\Imc_{E}\models \Kmc$. Let $\Delta^{\Imc_E}:=\{E(a)\in\mathbb{R}^d\;|\;a\;\text{occurs in}\;\Amc\}$ (where $d$ is the dimension of the embedding $E$), $a^{\Imc_E}:=\embedding(a)$ for every individual $a$ occurring in $\Amc$ and $r^{\Imc_E}:=\{(\vec{x},\vec{y})\in\Delta^{\Imc_E}\times\Delta^{\Imc_E}\;|\; \vec{x} \oplus \vec{y}\in E(r)\}$ for every $r\in\NR$ (recall that the language does not consider concept assertions and patterns use only roles). 
For every role assertion $r(a,b)\in\Amc$, observe that $\Imc_E\models r(a,b)$ iff $E(a)\oplus E(b)\in E(r)$ iff $E\models_{expr} r(a,b)$. 
Next, we consider TBox axioms in $\Tmc$.
\begin{itemize}
    \item If $r\sqsubseteq r^-\in\Tmc$ is a symmetry pattern, then $E\models_{\mi{expr}} r\sqsubseteq r^-$ implies that the region $E(r)$ is equal to its mirror-image w.r.t. the identity line, i.e. for all vectors $\vec{x},\vec{y}\in\mathbb{R}^d$, we have $\vec{x}\oplus \vec{y}\in E(r)$ iff $\vec{y}\oplus \vec{x}\in E(r)$. By definition of $r^{\Imc_E}$, it then follows that $\Imc_E\models r\sqsubseteq r^-$.
    \item If $r_1\equiv r_2^-\in\Tmc$ is an inversion pattern, then $E\models_{\mi{expr}} r_1\equiv r_2^-$ implies that $E(r_1)$ is equal to the mirror image of the region $E(r_2)$ w.r.t. the identity line. This means that $\vec{x},\vec{y}\in\mathbb{R}^d$, $\vec{x}\oplus \vec{y}\in E(r_1)$ iff $\vec{y}\oplus \vec{x}\in E(r_2)$ for all $\vec{x},\vec{y}\in\mathbb{R}^d$. By definition of $r_1^{\Imc_E},r_2^{\Imc_E}$, it follows that $\Imc_E\models r_1\equiv r_2^-$.
    \item If $r_1\sqsubseteq r_2\in\Tmc$ is a hierarchy pattern, then $E\models_{\mi{expr}} r_1\sqsubseteq r_2$ implies that $E(r_1)\subseteq E(r_2)$. Again, by definition of $r_1^{\Imc_E}$ and $r_2^{\Imc_E}$ we have that $\Imc_E\models r_1\sqsubseteq r_2$.
    \item If $r_1\sqcap r_2\sqsubseteq r_3\in\Tmc$ is an intersection pattern, then $E\models_{\mi{expr}} r_1\sqcap r_2\sqsubseteq r_3$ implies that $E(r_1)\cap E(r_2)\subseteq E(r_3)$. Again, by definition of $r_1^{\Imc_E},r_2^{\Imc_E},r_3^{\Imc_E}$ it follows that $\Imc_E\models r_1\sqcap r_2\sqsubseteq r_3$.
    \item If $r_1\sqsubseteq\neg r_2\in\Tmc$ is a mutual exclusion pattern, then $E\models_{\mi{expr}} r_1\sqsubseteq\neg r_2$ implies that $E(r_1)\cap E(r_2)=\emptyset$. Again, by definition of $r_1^{\Imc_E},r_2^{\Imc_E}$, it follows immediately that $\Imc_E\models r_1\sqsubseteq\neg r_2$.
    \item If $r_1\sqsubseteq\neg r_1^-\in\Tmc$ is an asymmetry pattern, then $E\models_{\mi{expr}} r_1\sqsubseteq\neg r_1^-$ implies that $E(r_1)$ is disjoint from its mirror image w.r.t. the identity line. That is, for all $\vec{x},\vec{y}\in\mathbb{R}^d$ we have that $\vec{x}\oplus \vec{y}\in E(r_1)$ implies that $\vec{y}\oplus \vec{x}\not\in E(r_1)$. Again, by definition of $r_1^{\Imc_E}$, it follows immediately that $\Imc_E\models r_1\sqsubseteq\neg r_1^-$. 
    
    \item If $r_1\circ r_2\sqsubseteq r_3\in\Tmc$ is a composition pattern and $\embedding\models_{\mi{expr}}r_1\circ r_2\sqsubseteq r_3$, by definition of the \ExpressivE semantics this means that $E(r_1\circ r_2)\subseteq\embedding(r_3)$, where $E(r_1\circ r_2)$ is the compositionally defined convex region of $r_1$ and $r_2$, which is such that, 
for every $\vec{u},\vec{v},\vec{w}\in\mathbb{R}^\dimension$, 
$\vec{u}\oplus \vec{v}\in\embedding(r_1)$ and $\vec{v}\oplus \vec{w}\in\embedding(r_2)$ iff $\vec{u}\oplus\vec{w}\in E(r_1\circ r_2)$. It follows immediately from the definition of $r_1^{\Imc}, r_2^{\Imc},r_3^{\Imc}$ that $\Imc_E\models r_1\circ r_2\sqsubseteq r_3$.   
\qedhere
\end{itemize}
\end{proof}

\begin{proposition}[Absence of completeness, faithfulness, entailment closure, full KB-expressiveness]
\ExpressivE is not complete for any language that contains role assertions and a mutual exclusion pattern.

Hence (by Theorem~\ref{thm:propertiesfinite}), it is not able (nor guaranteed) to be weakly or strongly KB-, TBox- or ABox-faithful, it is not able (nor guaranteed) to be KB-, TBox- or ABox-entailed, and it is not fully KB-expressive.
\end{proposition}
\begin{proof}
Example \ref{ex:expressive-dont-inject-exclu} shows that \ExpressivE is not complete for a language that contains a mutual exclusion pattern.
\end{proof}

\begin{proposition}[Full ABox-expressiveness (\cite{ExpressivE}, Theorem 5.1)]
\ExpressivE is fully ABox-expressive.
\end{proposition}

\begin{proposition}[Full TBox-expressiveness (\cite{ExpressivE}, Theorem 5.2)]
\ExpressivE is fully TBox-expressive for the language of \emph{positive} patterns.
\end{proposition}
\begin{proof}
The result is a consequence of Theorem 5.2 in \cite{ExpressivE}. We do not obtain it for the full language of patterns because  
 the authors consider a notion of `capturing exclusively' which is slightly weaker than 
 Definition \ref{def:capturingpatterns} in that faithfulness is only required w.r.t. \emph{positive} (or negation-free) patterns 
 (that is, no mutual exclusion and no asymmetry).
 \end{proof}

\end{document}